\documentclass[11pt,letterpaper]{article}
\usepackage[a4paper,bindingoffset=0.2in,%
left=0.8in,right=0.8in,top=1in,bottom=1in,%
footskip=.25in]{geometry}
\usepackage[ruled,linesnumbered]{algorithm2e}
\usepackage[utf8]{inputenc}
\usepackage{amsmath, nicefrac}
\usepackage{amsthm}
\usepackage{amssymb,verbatim}
\usepackage{mathtools, bbm, xcolor}
\usepackage{algorithmic}
\usepackage{tabularx}

\newtheorem{theorem}{Theorem}[section]
\newtheorem{lemma}[theorem]{Lemma}

\newtheorem{assumption}{Assumption}[section]
\newtheorem{example}{Example}[section]
\newtheorem{remark}{Remark}[section]
\newtheorem{property}{Property}[section]

\usepackage{graphicx}
\usepackage{booktabs}
\usepackage{url}
\usepackage{cite}
\usepackage[title]{appendix}
\usepackage[utf8]{inputenc} 
\usepackage[T1]{fontenc}    
\usepackage{hyperref}       
\usepackage{url}            
\usepackage{booktabs}       
\usepackage{amsfonts}       
\usepackage{nicefrac}       
\usepackage{microtype}      

\usepackage{authblk}
\usepackage{breakcites}
\usepackage{amsmath}
\usepackage{bbm}
\usepackage{verbatim}
\usepackage[font=small]{caption}
\usepackage{amsthm}
\usepackage{subcaption}
\usepackage{xcolor}

\usepackage{subfiles}
\usepackage[inline]{asymptote}

\DeclareMathOperator{\E}{\mathbb{E}}
\let\Pr\relax\DeclareMathOperator{\Pr}{\mathbb{P}}

\DeclareMathOperator{\Proj}{\text{Proj}}

\DeclareMathOperator*{\argmax}{arg\,max}
\DeclareMathOperator*{\argmin}{arg\,min}
\DeclareMathOperator{\Tr}{Tr}
\allowdisplaybreaks
\title{Stochastic Linear Bandits with Protected Subspace}

\author[1]{Advait Parulekar}
\author[2]{Soumya Basu}
\author[3]{Aditya Gopalan}
\author[4]{Karthikeyan Shanmugam}
\author[1]{Sanjay Shakkottai}
\affil[1]{Department of ECE, The University of Texas at Austin}
\affil[2]{Google}
\affil[3]{Department of ECE, Indian Institute of Science}
\affil[4]{IBM AI Research}

\begin{document}

\maketitle
\setlength\parindent{0pt}
\begin{abstract}
  We study a variant of the stochastic linear bandit problem wherein we optimize a linear objective function but rewards are accrued only orthogonal to an unknown subspace (which we interpret as a \textit{protected space}) given only zero-order stochastic oracle access to both the objective itself and protected subspace. In particular, at each round, the learner must choose whether to query the objective or the protected subspace alongside choosing an action. Our algorithm, derived from the OFUL principle, uses some of the queries to get an estimate of the protected space, and (in almost all rounds) plays optimistically with respect to a confidence set for this space. We provide a $\tilde{O}(sd\sqrt{T})$ regret upper bound in the case where the action space is the complete unit ball in $\mathbb{R}^d$, $s < d$ is the dimension of the protected subspace, and $T$ is the time horizon. Moreover, we demonstrate that a discrete action space can lead to linear regret with an optimistic algorithm, reinforcing the sub-optimality of optimism in certain settings. We also show that protection constraints imply that no consistent algorithm can have a minimax regret smaller than $\Omega(T^{3/4}).$ We empirically validate our results with synthetic and real datasets.
  
\end{abstract}

\section{INTRODUCTION}
 Consider the task of treating a disease characterized by some outlying biological marker. Often the medication necessary for treatment causes adverse side effects on other biological functionalities. During treatment, it is important to monitor such undesirable side effects by conducting various medical tests, while augmenting it with other medications to alleviate these adverse effects, and jointly calibrating the dosage of all these medications. Conducting tests may be expensive, thus it is desirable to find a treatment that has no side effects with efficient tests to optimally affect only the desired bio-marker. Such concerns are widespread in the treatment of disease - patients often receive multiple medications and the mitigation of drug related problems is a common concern, especially in the presence of comorbidities \cite{Hasniza2020Diabetes}. Optimal blood pressure control, for instance, is described as a challenge in the treatment of type 2 diabetes \cite{Hasniza2020Diabetes}, and antipsychotics prescribed for schizophrenia can result in 
 side effects such as obesity, dyslipidemia and type 2 diabetes \cite{mackenzie2018Schizophrenia}. Combination therapy (where a variety of medications are jointly prescribed) is often used to reduce the impact of adverse effects \cite{garcia2018combination}. Our work abstractly considers the problem of finding an optimal combination therapy guided by sequential 
 testing during the course of a patient's treatment to ensure recovery with the least cumulative side effects.

We approach this problem as an online decision making problem in which the results of various tests of bio-markers are regarded as bi-linear functions of treatments and patient characteristics. At each round the physician may take an action (specify a therapy, dosages, schedules, etc.) $A_t$ chosen from some given set  $\mathcal{A}_t\subset \mathbb{R}^d$. The physician has access to a test to monitor the result of the therapy, the result of which is given by $X_t = \langle \theta_0, A_t\rangle+\eta_t$ with $\eta_t$ representing some sub-gaussian noise, for some $\theta_0\in \mathbb{R}^d$. There may be other tests which should not be affected by the therapy (these test for side effects). Such tests are represented by $\theta_i\in \mathbb{R}^d, i\in [L]$, and their outcomes are similarly sub-gaussian with mean $\langle \theta_i, A_t\rangle$. In this setting, while the expected \textit{feedback} from an action $A_t$ is given by $X_t$ above, the learners \textit{reward} depends only on the components orthogonal to the protected space. That is, the reward given by $\langle a_t, \theta_\perp\rangle$ where $\theta_\perp$ is the component of $\theta_0$ orthogonal to $\{\theta_i\}_{i\in [L]}$ is unseen. In some sense, this is the component of the therapy that does not contribute to side-effects. 

The objective is to minimize \textit{pseudo-regret}, which is the total difference (that is, summed over all the rounds) between the \textit{expected} reward obtained by a genie who knows the means of the outcomes of every test exactly for each therapy, and the learner. 

Surprisingly, despite a similarity to the standard stochastic linear bandit problem, the partial information model makes this problem considerably more difficult. A key property that is used to upper bound regret in the linear bandit model is that under linear transformations, subgaussian random variables remain subgaussian. Speaking broadly, this allows us to use Hoeffding-style bounds to get confidence sets for unknown parameters under subgaussian assumptions of the noise. However our rewards are not linear functions of the unknown parameters, thus we require additional techniques to propagate estimates on confidence sets as samples are adaptively acquired over time.

\subsection{Contributions}
\label{section:contrib}

\begin{itemize}
    \item[1.] \textbf{Model:} We introduce the protected linear bandit as a model for online decision making with incomplete bandit feedback in which some subspace is considered to be protected, meaning that projections onto that space are subtracted from our reward. The optimal action is thus not the one that aligns most with the target vector, but rather the one which aligns most with the component of the target vector orthogonal to the unknown protected subspace. It is important to note that we do not have direct access to these projections, but only to inner products with some fixed set of individual vectors in the subspace.
    
    \item[2.] \textbf{Algorithm and Regret Upper Bounds:} We propose an algorithm (Algorithm \ref{algorithm:protectedLUCB}) for the above and derive an upper bound for its regret that grows as $\tilde{O}(sd\sqrt{T})$ in the number of rounds $T$, similar to the best possible linear bandit regret for the case when the action space is the unit ball. The algorithm consists of two parts. First we remove redundancy in the set of protected vectors with a uniform exploration phase. We then restrict our attention to this independent set of constraints and play optimistically using an \textit{upper confidence bound} based algorithm. Typically in OFUL, confidence regions are maintained around the unknown parameters of interest. However in this case, the projection operator on the protected space is the object of interest, and it is unclear what it even means to have a confidence region for this object. To circumvent this, we only maintain individual confidence ellipsoids around an observed basis that spans the protected subspace. 
    Translating sub-gaussian tail based confidence regions on the basis vectors to an appropriate confidence region on the projection operator involves a non-linear transformation and hence destroys sub gaussianity, so instead we directly construct a confidence interval on the reward {\em only} for the optimistic action (see Section \ref{section:key_difficulties}).


    \item[3.] \textbf{Regret Lower Bounds:} This new model comes with an interesting difficulty: even if we play an action infinitely many times, observing any number of noisy inner products with all of the protected vectors, we may not be able to find a good high probability confidence interval for the reward from that action. For general action spaces, Example (\ref{example:linear_regret}) shows an instance where naive optimism can lead to linear regret with this partial feedback model, and in Section \ref{section:lower_bound} we show a $\Omega (T^{\frac{3}{4}})$ regret lower bound for any algorithm on a finite (time-varying) action space.
\end{itemize}

\textbf{Notation}
We will denote by $\Proj_{\{\theta_i\}_{i\in [L]}}$ the projection operator onto the space spanned by $\{\theta_i\}_{i\in [L]}$ and by $\Proj_{\{\theta_i\}_{i\in [L]}}^{\perp}$ the projection onto the orthogonal subspace. We use $[L]=\{1, 2, \cdots, L\}$ to denote the set of the first $L$ integers. We use $||x||_V$ to refer to the weighted norm $\sqrt{x^\top Vx}$. Given a matrix $P\in \mathbb{R}^{d\times L}$ (respectively, vector $x\in \mathbb{R}^{L}$) and a set $S\subseteq [L]$, we denote by $P_S\in \mathbb{R}^{d\times |S|}$ (respectively, $x_S$), the submatrix (vector) whose columns are the ones in $P$ ($x$) indexed by $S$. We denote the minimum non-zero eigenvalue by $\lambda_{\min}(\cdot)$. We denote by $\mathcal{B}_2^d$ the unit $2$-norm ball in $\mathbb{R}^d$. A notation table is given in Appendix~\ref{appendix:notation}.

\section{RELATED WORK}\label{section:related_work}
Multi armed bandits have been studied for decades, at least since \cite{Robbins1952Bandits}, and optimism in the face of uncertainty (OFUL) has proved to be an effective strategy in low regret algorithm design~\cite{Auer2002MAB,yadkori2011OFUL}. We point an interested reader to \cite{LatSze20} and references therein for works which are not directly related to ours.

Linear Bandits, where observable rewards are generated as the noisy inner products of actions and a hidden vector,  were analyzed in \cite{dani2008linbandits, yadkori2011OFUL} and the regret of an optimistic algorithm was shown to grow as $O(\sqrt{T}\log T)$ depending only on the dimension of the representation, \textit{independently} of the number of arms. In our model, additional to the hidden vector (as in linear bandits) we have a hidden protected space (spanned by multiple hidden constraint vectors). Our reward is the inner product of action and the component of the hidden vector orthogonal to the hidden protected space. Further, we do not observe the reward directly, instead we are allowed to make partial queries which, for diverse enough action space, can be used to infer the optimal action. Bandits with indirect access to rewards are studied under partial monitoring with finite~\cite{bartok2014PM, lattimore2019cleaning}, and infinite~\cite{bartok2014PM} action spaces. However, inferring optimal action in our model requires use of additional structure which absent in \cite{bartok2014PM}. 

From a motivational standpoint, we share similarities with linear bandit with safety constraints where a learner is required to be {\em safe}. In \cite{amani2019safety}, the authors study a linear bandit with {\em known linear constraints} where the actions should not violate these constraints. They propose an optimistic algorithm with initial safe exploration. This setting has been studied extensively, through the design of Thompson sampling based techniques~\cite{moradipari2020safe}, and extension to safe generalized linear bandits~\cite{amani2020generalized}, safe contextual bandits~\cite{daulton2019thompson}, and safe reinforcement learning~\cite{hasanzadezonuzy2020learning}. In a different model, \cite{kazerouni2017conservative} studied online learning where regret is constrained to be small compared to a {\em known baseline}. We differ from these works technically, as the constraints are unknown to us, unlike the above works. Further, we consider the protected space as reward shaping parameters, rather than {\em hard} constraints.  Additionally, in the probabilistically approximately correct (PAC) learning framework, safety constrained optimization with unknown constraints and objectives with access to zeroth order oracles is studied in another line of work  \cite{usmanova2019log, usmanova2019safe,fereydounian2020safe}. However, the convergence results in PAC-learning framework do not translate into regret minimization directly, as the former do not consider balancing exploration and exploitation.
We expand on the connections with Linear Bandits, Safety Constrained Linear Bandits, and Partial monitoring in Section~\ref{section:model_comparison}.

\section{MODEL}\label{section:model}
We consider a game between a player and a stochastic environment in which we have query access to $L+1$ unknown vectors $\theta_0, \theta_1, \cdots, \theta_L \in \mathbb{R}^d$ with $||\theta_i||_2\le M$ for all $i$. The vectors $\theta_1, \theta_2,\cdots \theta_L$, the \textit{protected} vectors, are provided such that they span the protected subspace. In the context of our motivating problem, these represent low dimensional linear embeddings of the various tests for the biomarkers associated with side-effects. We are given a large number, $L$, of them, however they may represent a lower dimensional protected subspace of $\mathbb{R}^d$. 
We refer to $\theta_0$ as the target vector. We would like to play arms that align as well as possible with $\theta_\perp = \Proj_{\{\theta_i\}_{i\in [L]}}^{\perp}\theta_0$, that is, the orthogonal projection of $\theta_0$ onto the subspace orthogonal to the protected subspace. In the absence of the protected vectors, this would just be a stochastic linear bandit problem parameterized by $\theta_0$.

At every round $t$, the player can choose any action $A_t\in \mathcal{A}_t$, and an index $I_t\in \{0\}\cup [L]$, and receive a corresponding feedback of $X_t = \langle A_t, \theta_{I_t}\rangle +\eta_t$ where $\eta_t$ is a conditionally $R$-subgaussian zero-mean noise.

\textbf{Regret: }The sub-optimality of action $a\in \mathcal{A}_t$, $\Delta_a$, is given by $$\Delta_a = \langle a^*_t-A, \Proj_{\{\theta_i\}_{i\in [L]}}^{\perp}\theta_0\rangle$$ where $$a^*_t=\argmax_{a\in \mathcal{A}_t} \langle a, \Proj_{\{\theta_i\}_{i\in [L]}}^{\perp}\theta_0\rangle$$ is the optimal action. The goal is to minimize pseudo-regret with respect to a genie who is aware of the true vectors $\{\theta_i\}_{i\in \{0\}\cup [L]}$ (and so would play $a^*_t$ at each round):
$$\mathcal{R}_{[T]} = \sum_{t\in [T]}\Delta_{A_t} = \sum_{t\in [T]} \langle a^*_t-A_t, \Proj_{\{\theta_i\}_{i\in [L]}}^{\perp}\theta_0\rangle.$$

\textbf{Assumptions: }We now discuss the assumptions we make and motivations for them.
\begin{assumption}\label{assumption:action_space}
The action space $\mathcal{A}_t$ at all times consists of all vectors with unit norm, i.e. $\mathcal{A}_t =\mathcal{B}^2_d$  . 
\end{assumption}
This assumption is helpful due to the nature of the reward function. Finite action spaces with optimistic algorithms can sometimes lead to problems such as the one in Example \ref{example:linear_regret}. In fact, we show in Section \ref{section:lower_bound} that a particularly bad action space \textit{must} result in $\Omega(T^{\frac{3}{4}})$ regret for a consistent algorithm. What we really require is that any vector we desire to play as an action be available to us in the actions space. In the setting in which an action corresponds to a therapy (as in the example of the introduction), this just means that the physician is able to decide upon a therapy rather than prescribe one from a predetermined set.

Because of Assumption \ref{assumption:action_space}, the optimal action $a^*_t$ is just  $\frac{\Proj_{\{\theta_i\}_{i\in [L]}}^{\perp}\theta_0}{||\Proj_{\{\theta_i\}_{i\in [L]}}^{\perp}\theta_0||_2}$ at all times. This is the normalized projection of $\theta_0$ onto the space orthogonal to the protected subspace.
\begin{assumption}
There exists a subset $S\in [L]$ of size $|S| = s$ such that $\lambda_{\min}(\sum_{i\in {S}}\theta_i\theta_i^\top )>0$, while any larger set $S'$ has $\lambda_{\min}(\sum_{i\in {S'}}\theta_i\theta_i^\top )=0$. We assume knowledge of $s$.
\end{assumption}
This says that there is a $s$ dimensional subspace that contains all of the protected vectors. Our regret bounds will be in terms of $s$ rather than $L$. We denote the greatest such $\lambda_{\min}(\sum_{i\in {S}}\theta_i\theta_i^\top )$ (over all choices of $S$ with $|S|=s$) simply as $\lambda_{\min}$. This corresponds to the best spanning set of protected vectors. If we instead know $\lambda_{\min}(\sum_{i\in {S}}\theta_i\theta_i^\top )$, we can remove this assumption and use Algorithm \ref{algorithm:CORE-SET_alternative} instead of Algorithm \ref{algorithm:CORE-SET}. This alternative is discussed in Appendix \ref{appendix:unknown_s}.

Let $\mathcal{F}_t = \sigma(A_1, A_2,\cdots, A_t, \eta_1,\eta_2,\cdots, \eta_t)$ denote the $\sigma$-algebra generated by all actions and noises up to and including time $t$. 
\begin{assumption}\label{assumption:noise}
The noise on the observed feedback, $\eta_t$, is conditionally zero-mean $R$-subgaussian, meaning $\E[\eta_t|\mathcal{F}_{t-1}] = 0$ and $\E[e^{\lambda\eta_t}|\mathcal{F}_{t-1}] \le e^{\frac{1}{2}\lambda^2R^2}$.
\end{assumption}
This is standard, and used for deriving concentrations for the confidence sets for the unknown parameters.

\subsection{Differences from Related Models}\label{section:model_comparison}

\textbf{Linear Bandits: }
The standard linear bandit problem considers minimizing regret while learning a single unknown vector \cite{dani2008linbandits}, \cite{yadkori2011OFUL} {\em without} other protected directions. In our setting, the regret depends on several unknown vectors; however, in each round, we only get a signal from one. As such, the noisy feedback that are observed from the player’s actions $\{X_s\}_{s\in [T]}$ do not immediately give us the sub-optimality of an action. When a player plays action $(A_t, I_t)$, it observes $X_t = \langle A_t, \theta_{I_t}\rangle+\eta_t$ and incurs regret $\langle a^*_t-A_t, \Proj_{\{\theta_i\}_{i\in [L]}}^{\perp}\theta_0\rangle$. In particular, the player does \textit{not} see a noisy version of $\langle A_t, \Proj_{\{\theta_i\}_{i\in [L]}}^{\perp}\theta_0\rangle$. Aside from choosing the arm to pull, a player must also choose which vector to query with that arm. The analysis is further obscured by the fact that the rewards are a non-linear function of the unknown parameters. Finally, letting the set of protected vectors be  empty  ($L = 0$) recovers the standard linear bandit, so our setting is a generalization.

\textbf{Safety-constrained Linear Bandits: }
Safety-constrained bandits, studied, for instance, in \cite{kazerouni2017conservative}, \cite{amani2019safety}, are typically supposed to guarantee a safety constraint with high probability at each round. For instance, \cite{kazerouni2017conservative} require that the cumulative regret of a learner not exceed the regret of a baseline learner by more than a small multiplicative factor. \cite{amani2019safety} have a safety constraint that is a geometric constraint on the arms that can be played at each round. Aside from maximizing the cumulative regret against $a_t^\top \theta_0$, they have a known matrix $B$ and known constant $c$ such that the arm they pull at each round $a_t$ must satisfy $a_t^\top B\theta_0<c$ with high probability for some safety threshold $t$. Both of these are essentially constraints on the exploration of a learner. In contrast, we do not enforce any explicit exploration constraint. Rather, the difficulty of our problem is to \textit{learn} the safety constraints simultaneously with the objective. Moreover, the aforementioned works (i) typically consider a single safety constraint as opposed to multiple, unknown directions $\{\theta_i\}_{i \in L}$, and (ii) they crucially assume `free' access to an observation of the constraint violation at each action round, leading to very rapid learning of the linear constraint halfspace; in our setting, the exploration of the constraint/protection is partial (learn about one of the $\theta_i$) and has to be adaptively decided. 

\textbf{Linear Partial Monitoring: }
 A reduction to the linear partial monitoring framework in~\cite{kirschner2020information}, although possible, results in linear regret with existing guarantees. ~\cite{kirschner2020information} provide a regret spectrum based on how informative the action space is, and derive a linear minimax bound for regret on games that are not \textit{globally observable}. The following is a reduction to the linear partial monitoring setting. 
 
 Let $\theta_\perp=\Proj_{\{\theta_i\}_{i\in [L]}}^{\perp}\theta_0$. We may take $\theta = e_{L+1}\otimes \theta_\perp + \sum_{i=0}^L e_i\otimes \theta_i$. An action $(i, a)\in [L]\times \mathcal{A}$, is encoded as $\textbf{e}_{L+1}\otimes a$, while $A_{(i, a)}$ is taken to be $e_i\otimes a$. The partial monitoring game described here is not globally observable, hence gives linear regret, since for all $a_1, a_2\in \mathcal{A}$, we have $\textbf{e}_{L+1}\otimes (a_1-a_2) \not\in \text{Span}_{i\in [L], a\in \mathcal{A}}A_{(i, a)}$. Here $\otimes$ refers to the Kronecker product.
 
To overcome this difficulty, we leverage crucially the structure in $\theta$ 
(specific to our problem), that the first $d$ coordinates of $\theta$ are actually a known function of the last $(L+1)d$ coordinates (a projection). 
 
\section{PROTECTED LIN-UCB}\label{section:confidence_sets}
In this section, we present an algorithm for the regret minimization problem described in Section~\ref{section:model}.
Our algorithm, Algorithm~\ref{algorithm:protectedLUCB},  is developed following the Optimization in the Face of Uncertainty (OFU) principle  \cite{yadkori2011OFUL}, where we play optimistic actions that maximize the reward with high probability. For that purpose, we maintain and continually refine respective high probability confidence sets for a subset of protected vectors that spans the protected space, namely the {\em core set}. As the dimension of the protected space is assumed to be known to be $s$, it is possible to find a set of $s$ protected vectors that span the space, and any additional vectors need not be considered. In the first phase of the algorithm, we use Algorithm~\ref{algorithm:CORE-SET} to reduce the number of relevant unknown vectors in an approximately optimal way.

The method described in \cite{yadkori2011OFUL} to construct confidence sets, which we will use, is as follows. After $t$ rounds, suppose we have queried $\theta$ with arms $\{A_s\}_{s\in [t]}$ and received feedback $\{X_s = \theta^\top  A_s+\eta_s\}_{s\in [t]}$. We use these to determine the regularized maximum likelihood estimate (for regularizer $\rho$)
\begin{equation}\label{eqn:MLE}\hat{\theta}_t = (\sum_{s\in [t]}A_sA_s^\top  +\rho I)^{-1}(\sum_{s\in [t]}A_sX_s)
\end{equation}
In our setting, if the actions $a\in \mathcal{A}_t$ also satisfy $||a||_2\le M$, then Theorem 2 of \cite{yadkori2011OFUL} establishes that with probability $1-\delta$, $\theta_i$ lies in the set $\Theta_i$ defined as
\begin{equation}\label{eq:confidence_set}
\Theta_i = \{\theta : ||\hat{\theta}_i-\theta||_{V_i}\le \sqrt{\beta_{T_i}}\}
\end{equation}
where
$T_i = \sum_{s\le t}\mathbbm{1}_{I_s=i}$ is the number of times we sample $\theta_i$, $V_i = \sum_{s\le t}\mathbbm{1}_{I_s = i}A_sA_s^\top  +\rho I$, and $
\sqrt{\beta_t} = R\sqrt{d\log\left(\frac{1+tM^2/\rho}{\delta}\right)}+\sqrt{\rho}M$.
We refer to this set $\Theta_i$ as the confidence set for each unknown $\theta_i$. The dependence on $t$ is implicit - these confidence intervals generally shrink over time as learn about the unknown parameters, and at each time $t$, there is a well defined $\Theta_i$ to correspond to each unknown in the way prescribed above.

{\bf Coreset Estimation:} Because we need only concern ourselves with a spanning set of protected vectors, we first use the CORE-SET procedure to prune the set of protected vectors. We cannot simply pick $s$ of the protected vectors arbitrarily, as these may not span the whole protected space, and even if they do, they may span the space inefficiently. We do this with a deterministic, isotropic phase in which we sample every unknown vector uniformly in every direction in a round robin manner until we are certain that some subset is within a multiplicative factor of being optimal.

\begin{algorithm}[th!]
   \caption{CORE-SET for rank $k$}
   \label{algorithm:CORE-SET}
\begin{algorithmic}
\STATE $\{e_i: 1\leq i\leq d\} \gets$  the standard basis\;
\STATE $t\leftarrow 1$\;
\WHILE{$\forall S\subseteq [L]$, $|S| = k$, and\\ $\lambda_{\min}(\sum_{i\in S}\hat{\theta}_i\hat{\theta}_i^\top )\mathtt{\le}\frac{16LR(M+R)(d\log 6+\log\frac{1}{\delta})}{\sqrt{t}}$}
    \FOR{$i\in [d]$}
        \FOR{$p\in [L]$}
            \STATE Play $(e_i, \theta_p)$, observe feedback $x$\;
            \STATE Update $\Theta_i$ with $(e_i, x)$
            following Eq.\ref{eqn:MLE} and Eq.\ref{eq:confidence_set}\;
        \ENDFOR
    \ENDFOR
    \STATE $t\leftarrow t+1$\;
\ENDWHILE
\STATE return $(\argmax_{S \subseteq [L], |S|=k} \lambda_{\min}(\sum_{i\in S}\hat{\theta}_i\hat{\theta}_i^\top ), t)$
\end{algorithmic}
\end{algorithm}
From this we get a set $\tilde{S}$ for which with high probability, we have $$\lambda_{\min}(\sum_{i\in \tilde{S}} \theta_i\theta_i^\top )\ge \frac{1}{3}\max_{S'\in [L]}\lambda_{\min}(\sum_{i\in S'} \theta_i\theta_i^\top ).$$ This is our notion of being optimal within a multiplicative factor. We restrict our attention to this set.

Note that this phase only occurs once per instance of the problem - meaning that once we know which protected vectors are representative, we need not learn anything about the others. In the context of our motivation (treatment of disease while reducing the impact of adverse effects), this corresponds to picking a set of tests in advance for a particular ailment. The fact that we are doing this only once per ailment and not once per patient might alleviate ethical concerns related to providing experimental (exploratory) treatments to patients. Furthermore, this phase only introduces a constant to the regret, that is, the regret contribution of CORE-SET does not depend on $T$.

{\bf Protected LinUCB:}
\begin{algorithm}[th!]
   \caption{Protected LinUCB}
   \label{algorithm:protectedLUCB}
\begin{algorithmic}
    \STATE {\bfseries Input:} protected subspace dimension $s$
    \STATE $\tilde{S}, t_0\leftarrow$ CORE-SET(s)\;
    \STATE $t\leftarrow t_0$\;
    \STATE $V_i=\rho I_d$
    \FOR{$t \in [T]$}
    \STATE $(A_t, \{\overline{\theta}\}_i) = \argmax\limits_{a\in \mathcal{A}_t, \{\tilde{\theta}_i\in \Theta_i\}_{i\in \tilde{S}}, \tilde{\theta}_0\in \Theta_0} \langle a, \Proj_{\{\tilde{\theta}_i\}_{i\in \tilde{S}}}^{\perp}\tilde{\theta}_0 \rangle$\; \label{line:optimization}
    \STATE $I_t = \argmax_{i\in \tilde{S}} ||A_t||_{V_I^{-1}}\sqrt{\beta_{T_I}}$\;
    \STATE Play $(A_t, I_t)$, observe $X_t$\;
    \STATE Update $\Theta_{I_t}$ with $(A_t, X_t)$ following Eq.\ref{eqn:MLE} and Eq.\ref{eq:confidence_set}\;
    \STATE $V_{I_t}\leftarrow V_{I_t}+A_tA_t^\top$\; 
    \STATE Increment $T_{I_t}$\;
    \ENDFOR
\end{algorithmic}
\end{algorithm}
For all $i\in \tilde{S}$, we maintain one such ellipsoid $\Theta_i$ centered at $\hat{\theta}_i$ for each of the unknown vectors in the manner of the OFUL lin-UCB algorithm from \cite{yadkori2011OFUL}. We use these to infer a confidence interval for $\langle a_t, \theta_\perp\rangle$. These sets are such that each of the unknown vectors is contained within their respective confidence sets at every round with high probability. We refer to \cite{yadkori2011OFUL} for a detailed discussion on how such confidence sets are constructed. To keep track of the exploration for each unknown $\theta_i$ until time $t$, we denote by $T_i$ the number of times we have queried vector $i$, $T_i=\sum_{s\le t} \mathbbm{1}_{I_s = i}$ and set $V_i=\rho I+\sum_{s\le t}\mathbbm{1}_{I_s=i}A_sA_s^\top $. We then play optimistically with respect to these confidence sets. Concretely, we maximize over all actions $a\in \mathcal{A}_t$ and all possible $\overline{\theta}_i\in \Theta_i, i\in \tilde{S}$ and $\overline{\theta}_0\in \Theta_0$ the value of $\langle a, \Proj_{\{\overline{\theta_i}\}_{i\in \tilde{S}}}^{\perp}\overline{\theta}_0\rangle$. 
Note that the confidence set for $\theta_\perp$ is not a geometric ellipsoid, and characterizing its shape exactly is quite difficult (see also Section~\ref{section:key_difficulties}).

In each round a player must also chose an index determining the particular protected vector to be queried, and we make this decision based on which vector is least explored in the direction of the selected action.


\section{REGRET UPPER BOUND FOR $\mathcal{A}_t=\mathcal{B}_2^d$}\label{section:regret_bound}
In this section, we derive an upper bound on the regret of Algorithm (\ref{algorithm:protectedLUCB}). The algorithm begins by constructing a core-set $\tilde{S}$ of the protected vectors that optimally span the protected subspace. This core-set has cardinality $|\tilde{S}| = s$, the known dimension of the protected subspace and is constructed by paying a constant exploratory regret. Here we assume $M, R, \rho = 1$, but the results presented in the appendix are such that the dependence on these parameters is explicit.

We have the following theorem that allows us to get a spanning set of protected vectors that span the protected space near-optimally.
\begin{theorem}\label{thm:CORE-SET}
CORE-SET terminates in at most $t_0 = \nicefrac{2304L^2\big(d\log 6+\log\frac{L}{\delta}\big)^2}{\lambda_{\min}^2}$ iterations of the outer loop and returns a subset $\tilde{S}$ such that, with probability at least $1-\delta$, $\lambda_{\min}(\sum_{i\in {\tilde{S}}}\theta_i\theta_i^\top ) \mathtt{\ge} \frac{\lambda_{\min}}{3}$.
\end{theorem}
\begin{proof}[Proof sketch]
We establish error bounds on the protected vectors in Lemma \ref{lemma:uniform_lambda_bounds} and use these to bound the perturbation of the eigenvalues from a spanning set in Lemma \ref{lemma:PPT_perturbation}. (see Appendix~\ref{section:core-set_appendix} for details).
\end{proof}
Once the core-set is found, we play optimistically with respect to confidence sets derived from estimates that only include the core-set vectors, reducing the number of parameters we need to learn. We have the following high probability regret bound for Algorithm~(\ref{algorithm:protectedLUCB}):
\begin{theorem} \label{thm:known_subspace_dim}
If we have $\mathcal{A}_t=\mathcal{B}_2^d$, the regret of Algorithm \ref{algorithm:protectedLUCB} satisfies
\small
\begin{align*}
    \mathcal{R}_{[T]} &\le 12\sqrt{2}\frac{s+1}{\lambda_{\min}}\sqrt{Td\log(1+\frac{TL}{d})}\sqrt{\beta_T\left(\frac{\delta}{2(L+1)}\right)}\\&+\underbrace{\frac{4608L^3d\big(d\log 6+\log\frac{2L}{\delta}\big)^2}{\lambda_{\min}^2}}_{\text{CORE-SET estimation}}
\end{align*} with probability $1-\delta$ where $$\sqrt{\beta_t(\delta)}=R\sqrt{d\log(\tfrac{1}{\delta}+\tfrac{tM^2}{\delta\rho})}+M\rho^{\frac{1}{2}}.$$
\end{theorem}

\subsection{Key Difficulties}\label{section:key_difficulties}

We now describe the reasons we cannot straightforwardly apply results from linear bandit literature.

Given only stochastic zero-order access to vectors $\{\theta_i\}_{i\in \{0\}\cup [L]}$, we must play the arm $a\in \mathcal{A}_t$ which maximizes $\langle a, \Proj^{\perp}_{\{\theta_i\}_{i\in [L]}}\theta_0\rangle.$
Suppose for all $i \in L$, we know that the unknown vector $\theta_i$ was in some  confidence set $\Theta_i$ with high probability. Then, let the set of all possible $\theta_\perp$ be denoted $\Theta_\perp$ where each member is derived from a specific choice of $\{\theta_i\}_{i\in [L]}$ consistent with $\{\Theta_i\}_{i\in [L]}$. Clearly, this contains the true $\theta_\perp$ with high probability. Meanwhile, if we chose to play that action that gave us the maximum reward under any choice of $\theta_\perp\in \Theta_\perp$ then the sub-optimality of an action is \textit{upper bounded by the uncertainty in the mean reward for that action}, so a complete characterization of $\Theta_\perp$ would directly lead to a regret bound. 




However, explicitly constructing a high probability confidence set for $\theta_\perp$, denoted by $\Theta_\perp$, presents new problems. 
The key issue is that one cannot get meaningful confidence regions on the object of interest, namely the projection of $\theta_0$ given by $\Proj_{\{\theta_i\}_{i\in [L]}}^{\perp}\theta_0$. To see this,
observe that $\{\Theta_i\}_{i\in [L]}$ are confidence ellipsoids obtained from sub-gaussian tail bounds. However, the map $\{\theta_i\} \rightarrow \Proj_{\{\theta_i\}_{i\in [L]}}^{\perp} \theta_0 $ is not linear in $\theta_i$, and hence sub-gaussianity is not preserved through this transformation.


There is another way of seeing this difficulty. In standard linear bandit, for arm $a$ and the optimal parameter $\theta_0$, pulling arm $a$ repeatedly reduces our uncertainty of $\langle \theta_0, a\rangle$. However, the object of our interest is $\Proj_{\{\theta_i\}_{i\in [L]}}^{\perp}$, i.e. the {\em space} orthogonal to the protected vectors.
Thus, {\em (i)} the component of $a$ that lies in the protected space is not informative because any reduction in variance of a protected vectors in the span of the protected space does not change the variance of our estimate of the protected space, and {\em (ii)} the true reward depends on the protected vectors only through the space they span and not the vectors themselves. As such, it is not true that getting even infinite samples from an arm allows us to compute its mean reward with high confidence. Instances in which this fundamentally changes the regret bounds are presented in Sections \ref{section:failure} and \ref{section:lower_bound}. 

%
%


\subsection{Proof idea}
We consider the unknown linear operator $C$ that maps $\theta_i$ to $\overline{\theta}_i$ (the choice corresponding to the optimistic action) for each $i$ in the coreset, and replace  $\Proj_{\{\overline{\theta_i}\}}\theta_0=\Proj_{\{C \theta_i\}}\theta_0$ (the true projection of the target vector on the optimistic space) by $C\Proj_{\{\theta_i\}}\theta_0$ (here we have switched the order of $\Proj$ and $C$). Note that $C$ and $\Proj$ may not be commutative, but using self-adjointness and idempotence of projection operators allows us to do this for specifically the optimistic action when $\mathcal{A}_t=\mathcal{B}_2^d$ (actually all we need is that the optimistic vector at each step is available in the action space). We can now propagate the errors in the protected vectors \textit{linearly} through our estimates of the subspace, thus crucially {\em preserving sub-Gaussianity of noise.} 

Concretely, in Lemma \ref{lemma:delta_upperbound} we show an upper bound on the suboptimality of the player's action using Algorithm \ref{algorithm:protectedLUCB} as $$\Delta_{A_t}\le 2\underbrace{(3\frac{\sqrt{s}}{\lambda_{\min}}M+1)}_{(*)}||A_t||_{V_{i_t, t}^{-1}}\sqrt{\beta_{T}(\delta)}$$

Here the ($*$) multiplicative term comes from an online subspace estimation and can be thought of as a condition number for the operator $C$ above. We can use this to get a regret bound similarly to \cite{yadkori2011OFUL}.

\subsection{Remarks}

Here we discuss some of the key terms of the regret bound presented in Section~\ref{thm:known_subspace_dim}.

\begin{remark}[Comparison with OFUL algorithm for Lin-UCB in \cite{yadkori2011OFUL}]
The regret of the OFUL algorithm satisfies $$R^{L-UCB}_{[T]}\le 4\sqrt{Td\log(1+TL/d)}\sqrt{\beta_t(\delta)}$$ with probability $1-\delta$ where $$\sqrt{\beta_t(\delta)} = R\sqrt{d\log(\tfrac{1}{\delta}+\tfrac{TM^2}{\rho\delta})}+M\rho^{\frac{1}{2}}.$$
In comparison, our regret has a multiplicative $\frac{s+1}{\lambda_{\min}}$ factor. This comes from the fact that our rewards now depend on $s+1$ unrelated vectors. The dependence on $\lambda_{\min}$ comes from the way perturbations of vectors affect perturbations of the space they span.
\end{remark}

\begin{remark}[Knowledge of $s$]\label{remark:dimension_known}
If $s<L$, it is desirable to have regret that scales as $s$ and not $L$.
This raises an additional difficulty, as demonstrated by the following example.

Suppose in the first instance, $\theta_0=[1,1,1], \theta_1=[1,0,0], \theta_2=[1,0,0]$, while in the second $\theta_0=[1,1,1], \theta_1=[1,0,0], \theta_2=[1,\Delta, 0]$. The true subspace dimension in the first is 1, while in the second, it is $2$. The ideal action, $\theta_{\perp}$ is $[0,1,1]$ in the first, while it is $[0,0,1]$ in the second.

For small $\Delta$, it is difficult to decide between these, and deciding incorrectly leads to a sub-optimality that does not go to $0$ as $\Delta\rightarrow 0$. Note that this is very different from the analogous issue in the multi-arm bandit (MAB), where a separation of $\Delta$ leads only to a sub-optimality of $\Delta$. To further complicate matters, such a suboptimality in a MAB is addressed as directly as possible by sampling the relevant arms of the bandit. In our case, the separation is in a direction orthogonal to $\theta_{\perp}$, the direction we need to exploit.
%
\end{remark}


\subsection{A Failure of Optimism}\label{section:failure}
A study of this algorithm reveals an interesting phenomenon. While Theorem \ref{thm:known_subspace_dim} demonstrates a regret bound that scales in $T$ as $\tilde{O}(sd\sqrt{T})$ if we set the action space $\mathcal{A}_t$ to always be the unit ball $\mathcal{B}^d_2$, we also note in Theorem \ref{thm:hardness} that no consistent algorithm can do better than $\Omega(T^{\frac{3}{4}})$ with no restriction on the action space. In fact, the naive optimism of Algorithm \ref{algorithm:protectedLUCB} can get stuck with linear regret, as demonstrated in the following example.
\begin{example}\label{example:linear_regret}
Consider a problem with $d=2, L=1$, where $\theta_0 = u_{\frac{\pi}{4}}, \theta_1 = u_0$. For ease of notation, let $u_\alpha$ denote the point $(\cos \alpha, \sin\alpha)$. For simplicity, suppose the player knows $\theta_0$ exactly. Suppose that at all times the player is given the choice of actions $\mathcal{A}_t=\{a_1, a_2\}$ where $a_1=u_{\frac{\pi}{4}}$ and  $a_2=u_{\frac{\pi}{2}}$. Suppose at round $t$, $\theta_1$ has been queried $T_{1, t}$ times and the vector $\overline{\theta}_1=u_0+u_{-\frac{\pi}{4}}$ is in the confidence set for $\theta_1$, that is, $$||\overline{\theta}_1-\theta_1||_{V_{1, t}} = ||u_{-\frac{\pi}{4}}||_{V_{1, t}}\le \sqrt{\beta_{T_{1, t}}}.$$ Then an optimistic evaluation of $a_1$ is at least as good as the evaluation that uses $\overline{\theta}=u_0+u_{-\frac{\pi}{4}}$. With this as the protected vector, the evaluation of $a_1$ is $\cos^2 \frac{\pi}{8}$. Meanwhile, the evaluation of action $a_2$ can never exceed $\cos^2 \frac{\pi}{8}$. An optimistic player will play $a_1$ at round $t+1$. There is no hope of the player learning any better in the future, since $\overline{\theta}_1$ remains in the confidence ellipsoid
\begin{align*}
||\overline{\theta}_1-\theta_1||_{V_{1, t+1}}
&=||u_{-\frac{\pi}{4}}||_{V_{1, t+1}}\\
&=\sqrt{||u_{-\frac{\pi}{4}}||^2_{V_{1, t}}+\langle u_{-\frac{\pi}{4}}, u_{\frac{\pi}{4}}\rangle^2} \\
&= ||u_{-\frac{\pi}{4}}||_{V_{1, t}}\le \sqrt{\beta_{T_{1, t}}} \le \sqrt{\beta_{T_{1, t+1}}}
\end{align*}
and so the learner will just play $a_1$ again. Such a learner suffers linear regret under a naively optimistic policy.
\begin{figure}[ht!]
    \centering
	\includegraphics[width=0.15\linewidth]{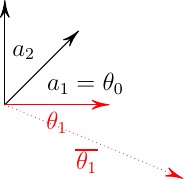}
	\caption{Instance described in Example \ref{example:linear_regret}}
\end{figure}
\end{example}


\section{REGRET LOWER BOUND FOR FINITE ACTION SPACE}\label{section:lower_bound}
In this section, we establish the difficulty of the protected linear bandit problem. Note that section (\ref{section:regret_bound}) provides a $O(\sqrt{T}\log T)$ upper bound on the regret of Algorithm \ref{algorithm:protectedLUCB} when the actions space is $\mathcal{B}_2^d$. We suggested in section \ref{section:confidence_sets} that an adversarial action space could make the problem much harder. Here we provide a lower bound for the regret of any algorithm on a specially chosen instance. 
\begin{theorem}\label{thm:hardness}
There is an instance of the Protected Linear Bandit problem such that any algorithm incurs a regret of $\Omega(T^{\frac{3}{4}})$.
\vspace{-1em}
\end{theorem}
\begin{proof}[Proof sketch]
Consider a pair of instances, denoted with superscripts $(1)$ and $(2)$. For both, we set our ambient space to have dimension $d=2$, and set $s=L=1$. We denote by $u_\alpha\in \mathbb{R}^2$ the vector $(\cos \alpha, \sin \alpha)$. Take $\alpha = T^{-\frac{1}{4}}$. We set $\theta^{(1)}_0=\theta^{(2)}_0=u_{\frac{\pi}{2}-\alpha}$. In instance $(1)$, we set $\theta^{(1)}_1=u_0$ while in instance $(2)$, we set $\theta^{(2)}_1=u_{-\alpha}$. In both instances, in each round, we allow the player an action space that consists of either the actions $\{u_{\pi-\alpha}, u_{2\alpha}\}$ or $\{u_{\pi-\alpha}, u_{2\alpha}, u_{\pi-3\alpha}\}$ with equal probability. These instances are chosen such that $u_{2\alpha}$ is always optimal for the second instance, while whenever $u_{\pi-3\alpha}$ is available, it is optimal for the first instance. The event in which $u_{\pi-3\alpha}$ is picked more than half the times it is available must thus have a high probability under the interaction between the algorithm with the first instance and a low probability in the interaction with the second instance. The Bretagnolle-Huber inequality \cite{LatSze20} allows us to control the maximum difference in this probability by the KL divergence induced by the different interactions, which we prove to be bounded by a constant. The complete proof is given in Appendix \ref{appendix:KL_upperbound}.
\end{proof}
\section{EXPERIMENTS}
\label{sec:experiment}

In this section, we validate our theoretical results with simulations on a synthetic instance, and an instance derived from the Warfarin dataset~\cite{WarfarinData} that consists of clinical and pharmacogenetic data on Warfarin dosage in the presence of other medications. 
For all experiments, we perform $10$ parallel runs, and report the cumulative regrets (average, and average $\pm$ 1 $\times$ standard deviation).

\textbf{Baseline Algorithms:}
Because this is a new model, there is no previously studied baseline that we are aware of. As mentioned in Section \ref{section:related_work}, the prior work on safety constrained Bandits \textit{requires} safety with high probability in each round, and assumes a known relationship between the target vector and the protected actions. We simulate against two natural baselines. Complete algorithms are listed in the Appendix.

\textbf{Round Robin LinUCB/LinUCB2: } Here we learn each of the protected vectors as separate instances of linUCB. We dedicate each round as a ``subspace learning" round with probability $\epsilon = \frac{1}{\sqrt{t}}$ (for Round Robin LinUCB2 we use $\epsilon = t^{-\frac{1}{4}}$) where $t$ is the round number, and iterate over the protected vectors playing exactly the OFUL algorithm of \cite{yadkori2011OFUL}. Otherwise we play the same arm as specified in Algorithm \ref{algorithm:protectedLUCB} but query the target vector $\theta_0$. Psuedocode is provided in Algorithm \ref{algorithm:RRLinUCB} in the Appendix. 

\textbf{$\epsilon$ greedy: }
Here with probability $\frac{\epsilon}{\sqrt{t}}$ the algorithm plays a uniformly random protected vector and a uniformly random arm. We use these samples to estimate (using MLE) the protected and target vectors, and otherwise play a pure exploitation strategy based on a subspace estimate derived from these MLE estimates. We manually optimize the hyper-parameter $\epsilon$. Psuedocode is provided in Algorithm \ref{algorithm:EG} in the Appendix. 

\textbf{Synthetic Data:} A problem instance was generated randomly by drawing vectors randomly from $\mathcal{N}(0, I_d)$ for $d=5$ in each round which are then normalized. We set $L=3$ and set $s=2$. We have set the regularization parameter $\lambda=0.1$ and the failure probability $\delta=0.001$. The regret due to the interaction of the player and the instance over $T=1000$ rounds and 4 times in parallel is plotted.

\begin{figure}[ht!]
    \centering
    \begin{subfigure}[b]{0.48\textwidth}
		\includegraphics[width=0.95\linewidth]{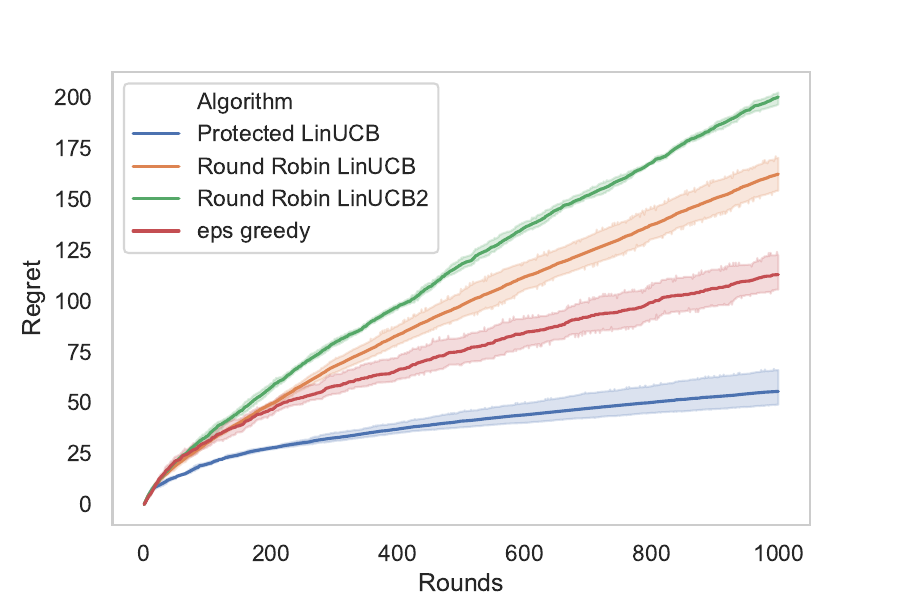}
        \caption{Regret of $\epsilon$-Greedy, and Algorithm~\ref{algorithm:protectedLUCB} with $\rho=0.1, \delta=0.001, R=0.001$. We have $s=2$, $L=4$, $d=6$, and $100$ arms randomly drawn on the unit sphere at each round.}
        \label{fig:synthetic_comparison}
	\end{subfigure}\hfill%
	\begin{subfigure}[b]{0.48\textwidth}
		\includegraphics[width=0.95\linewidth]{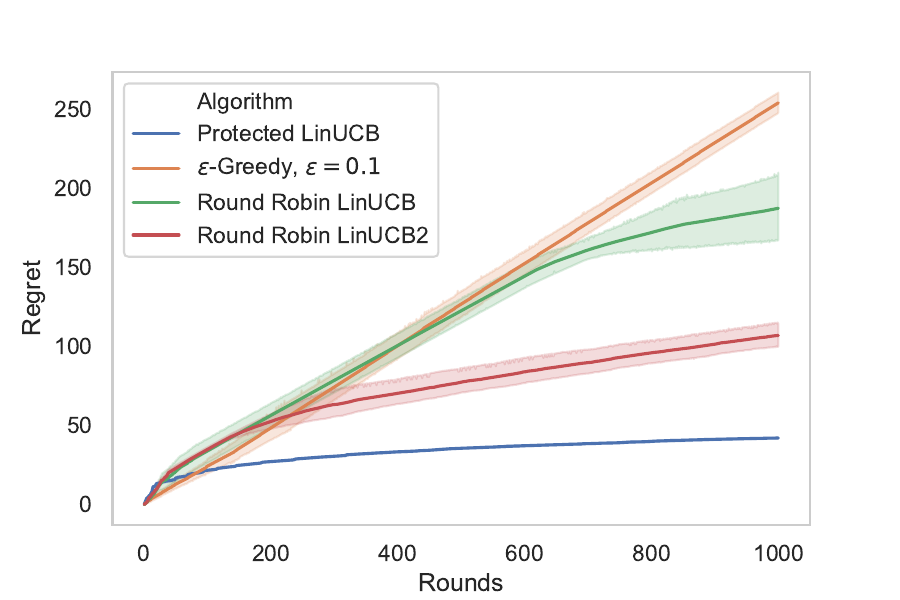}
        \caption{Regret of $\epsilon$-Greedy, and Algorithm~\ref{algorithm:protectedLUCB} with $\rho=0.1, \delta=0.001, R=0.001$ on Warfarin dataset. We have  $s=1$, $L=1$, $d=8$, and $1832$ fixed arms.}
        \label{fig:warfarin_comparison}
	\end{subfigure}%
\end{figure}

\textbf{Warfarin Dataset:}  We consider the Warfarin dataset~\cite{WarfarinData} and construct an instance to optimize Warfarin dosage in our setting. This dataset consists of dosages of Warfarin (an anticoagulant prescribed for Deep Vein Thrombosis, Stroke, Cardiomyopathy, etc) and other medications (`Simvastatin', `Atorvastatin', `Fluvastatin', etc.) as well as the resulting INR (International Normalized Ratio which indicates susceptibility to bleeding - this is provided as a number between roughly $1$ and $4$) and stability of Warfarin therapy (this is provided as a Boolean). 

In this context, we consider the task of optimizing a therapy consisting of some combination of these medications to get optimal Stability while minimally affecting deviation from the normal range of INR (defined to be 2.5). We model the therapy (combination of medications) as a unit norm vector in $a\in \mathcal{A}_t\subset \mathbb{R}^8$ (interpreted as the dosages of each of $8$ medications). We assume the following model, and learn the unknown parameters $\theta_0$ and $\theta_1$ from the data.
\begin{align*}
\text{INR} & \leftarrow \text{Subgaussian}(\theta_0^\top a, R) && a\in \mathcal{A}_t\\
\text{Stability} & \leftarrow \text{Bernoulli} (\theta_1^\top a) 
\end{align*}
We then construct a Protected Linear Bandit instance, where all the available therapy records comprise the action space (i.e. we interpret each therapy as an element of $\mathcal{A}$ (unchanging in time) which is large enough to approximate as $\mathbb{R}^d$, with $d=8$ and $1832$ elements (arms)), the INR test vector acts as the protected vector $\theta_1$ (i.e. $L=s=1$), and the Stability test vector acts as the reward vector $\theta_0$. We set $\rho=1$, and $\delta = 0.001$ in Algorithm~\ref{algorithm:protectedLUCB} and simulate the system for $5$ parallel runs each with $T=1000$ time steps.

\begin{remark}[Solving the optimization problem in Line \ref{line:optimization} of Algorithm \ref{algorithm:protectedLUCB}]
This is a maximization of a function that is not concave. In Appendix \ref{appendix:optimizer} we describe a simple way to solve this optimization explicitly for a fixed arm (that is, how to get the optimal $\overline{\theta}_0$ and $\overline{\theta}_i$ for a fixed $a_t$) if $\mathcal{A}_t=\mathcal{B}_2^d$. Even though $\mathcal{A}_t\ne \mathcal{B}_2^d$ for the above experiments, we use this optimizer as a heuristic for all of the algorithms above.
\end{remark}
\pagebreak

\bibliographystyle{plain}
\bibliography{sample}
\pagebreak
\section{Notation}\label{appendix:notation}
\begin{table}[h!]
\caption{Notation}
\begin{tabularx}{\textwidth}{@{}XX@{}}
\toprule
  $\theta_0$ & target vector\\ 
  $L$ & number of protected vectors\\
  $\{\theta_i\}_{i\in [L]}$ & unknown protected vectors \\
  $\{\Theta_i\}_{i\in 0\cup[L]}$ & confidence sets for $\{\theta_i\}$ constructed as in \ref{eq:confidence_set}\\
  $\hat{\theta_i}$ & maximum likelihood estimate for $\theta_i$ \\
  $\overline{\theta_i}$ & optimistic choice of $\theta_i$ \\
  $M$ & Upper bound on $\ell_2$ norm for unknown $\{\theta_i\}_{i\in 0\cup [L]}$ \\
  $R$ & Subgaussian norm of the noise in the feedback \\
  $\Proj_{\{\theta_i\}_{i\in [L]}}^{\perp}$ & \text{the projection onto the space orthogonal to $\{\theta_i\}_{i\in [L]}$}\\
  $\lambda_{\min}$ & \text{minimum non-zero eigenvalue operator}\\
\bottomrule
\end{tabularx}
\end{table}

\section{Proof of CORE-SET Estimation}\label{section:core-set_appendix}
For this section, because are choosing our matrix deterministically, we need not use the self-normalized bounds of \cite{yadkori2011OFUL}. We also need not use a regularization parameter because after a single round of querying a standard basis we will have an invertible $V_{i, t}$.

Let for some subset $S\subseteq [L]$, let $P_S=\sum_{i\in R}\theta_i\theta_i^T$, and let $\hat{P}_R=\sum_{i\in R}\hat{\theta}_i\hat{\theta}_i^T$. We denote by $\lambda_{\min}(P)$ the minimum eigenvalue of $P$.

\begin{theorem}\label{thm:CORE-SET_appendix}
Algorithm \ref{algorithm:CORE-SET} terminates in at most $\frac{576L^2R^2(M+R)^2\big(d\log 6+\log\frac{L}{\delta}\big)^2}{\lambda_{\min}^2}$ iterations of the outer loop and returns a subset $\tilde{S}$ such that, with probability at least $1-\delta$, $\lambda_{\min}(P_{\tilde{S}})\ge \frac{1}{3}\max_{S\subseteq [L], |R|=k}\lambda_{\min}(P_S)$.
\end{theorem}
\begin{lemma}\label{lemma:uniform_lambda_bounds}
    Suppose we sample each of the $L$ protected vectors using an orthonormal set of actions $T$ times for a total of $dLT$ isotropic samples. Then we have 
    $$||\hat{\theta}_i-\theta_i||_2 \le 2R\sqrt{2\frac{d\log 6+\log\frac{1}{\delta}}{T}}$$ and thus
    $$||(\hat{\theta}_i-\theta_i)(\hat{\theta}_i-\theta_i)^T||_2\le 8R^2\frac{d\log 6+\log\frac{1}{\delta}}{T}$$with probability at least $1-L\delta$ for every $i\in [L]$.
\end{lemma}
\begin{proof}
From (20.3) of \cite{LatSze20}, each estimate $\hat{\theta}_i$ of $\theta_i$ satisfies with probability at least $1- L\delta$
    \begin{align*}
    &||\hat{\theta}_i-\theta_i||_{V_{i, Td}}
    \le 2R\sqrt{2(d\log 6+\log\frac{1}{\delta})}\\
    \implies &\sqrt{\langle \hat{\theta}_i-\theta_i, V_{i, Td}\big(\hat{\theta}_i-\theta_i\big)\rangle} \le 2R\sqrt{2(d\log 6+\log\frac{1}{\delta})}\\
    \implies &\sqrt{T}||\hat{\theta}_i-\theta_i||_2 \le 2R\sqrt{2(d\log 6+\log\frac{1}{\delta})} && \text{by $V_{i, Td}=TI$ by construction}\\
    \implies &||\hat{\theta}_i-\theta_i||_2 \le 2R\sqrt{2\frac{d\log 6+\log\frac{1}{\delta}}{T}}\\
    \implies &||(\hat{\theta}_i-\theta_i)(\hat{\theta}_i-\theta_i)^T||_2\le 8R^2\frac{d\log 6+\log\frac{1}{\delta}}{T}&&\text{by $||vv^T||_2 = v^Tv$ for any column vector $v$}
    \end{align*}
\end{proof}
\begin{lemma}\label{lemma:PPT_perturbation}If we run CORE-SET for $T$ iterations of the outer loop, then for any $S\subseteq [L]$ we have
    $$||\hat{P}_S-P_S||_2\le8LR(M+R)\frac{d\log 6+\log\frac{1}{\delta}}{\sqrt{T}}$$ with probability $1-L\delta$.
\end{lemma}
\begin{proof}
    This follows from explicit lower bounds we get for exploration from CORE-SET. With probability $1-\delta$:
    \begin{align*}
        &||\hat{P}_S-P_S||_2\\
        & = ||\sum_{i\in [L]}(\hat{\theta}_i\hat{\theta}_i^T-\theta_i\theta_i^T)||_2\\
        & = ||\sum_{i\in [L]}(\hat{\theta}_i(\hat{\theta}_i-\theta_i)^T+(\hat{\theta}_i-\theta_i)\theta_i^T)||_2\\
        & \le ||\sum_{i\in [L]}\theta_i(\hat{\theta}_i-\theta_i)^T||_2+||\sum_{i\in [L]}(\hat{\theta}_i-\theta_i)\theta_i^T||_2\\
        &\quad\quad\quad\quad +||\sum_{i\in [L]}(\hat{\theta}_i-\theta_i)(\hat{\theta}_i-\theta_i)^T||_2\\
        & \le 2\sum_{i\in [L]}||\theta_i||_2||\hat{\theta}_i-\theta_i||_2 +||\sum_{i\in [L]}(\hat{\theta}_i-\theta_i)(\hat{\theta}_i-\theta_i)^T||_2&\text{by submultiplicativity 
        of $||\cdot||_2$ norm}\\
        & \le 2\sum_{i\in [L]}||\theta_i||_2||\hat{\theta}_i-\theta_i||_2+\sum_{i\in [L]}||(\hat{\theta}_i-\theta_i)(\hat{\theta}_i-\theta_i)^T||_2&\text{by triangle inequality}\\
        & \le 4\sum_{i\in [L]}||\theta_i||_2R\sqrt{2\frac{d\log 6+\log\frac{1}{\delta}}{T}}+\sum_{i\in [L]}8R^2\frac{d\log 6+\log\frac{1}{\delta}}{T}&\text{by Lemma \ref{lemma:uniform_lambda_bounds}}\\
        & \le 4LMR\sqrt{2\frac{d\log 6+\log\frac{1}{\delta}}{T}}+8LR^2\frac{d\log 6+\log\frac{1}{\delta}}{T}&\text{by $||\theta_i||_2\le M$}\\
        & \le 8LR(M+R)\frac{d\log 6+\log\frac{1}{\delta}}{\sqrt{T}}&\text{by $T, d, \frac{1}{\delta}, L\ge1$}
    \end{align*}
\end{proof}

We also have the following eigenvalue perturbation result.
\begin{lemma}\label{lemma:eig_perturbation}
    Let $\lambda_{\min}(P)$ denote the minimum eigenvalue of symmetric matrix $P$ with $P\in \mathbb{R}^{d\times d}$, and consider a symmetric noise matrix $E\in \mathbb{R}^{d\times d}$. Then $$\lambda_{\min}(P+E)\ge \lambda_{\min}(P)-||E||_2.$$
\end{lemma}

\begin{proof}
Let $\argmin_{v: \lVert v \rVert_2 =1} v^T (P+E) v = \hat{v}$
\begin{align}
    \lambda_{\mathrm{min}} (P+E) &= \min_{v: \lVert v \rVert_2 =1} v^T (P+E) v  \nonumber \\
    \hfill & =  \hat{v}^T P \hat{v} +  \hat{v}^T E \hat{v} \nonumber \\
    \hfill & \overset{a}{\geq}  \lambda_{\mathrm{min}} (P) +  \hat{v}^T E \hat{v} \nonumber \\
    \hfill & \overset{b}{\geq} \lambda_{\mathrm{min}} (P) -  \lVert E \rVert_2 \nonumber
\end{align}
(a): Definition of Rayleigh Quotient applied to the symmetric matrix $P$. (b): for any $v: \lVert v \rVert_2=1$, by Cauchy Schwartz, $ \lvert v^T E v \rvert \leq \lVert v \rVert_2 \lVert E v \rVert_2$. For any scalar $a$, $|a|< c$ for some $c>0$, then $a \geq -c$.
%
%
\end{proof}

\begin{proof}[Proof of Theorem \ref{thm:CORE-SET}]
We will use the shorthand $\alpha$ to denote the constant $8LMR(M+R)\big(d\log 6+\log\frac{1}{\delta}\big)$ in Lemma \ref{lemma:PPT_perturbation}, so that we have $||\hat{P}_S-P_S||_2\le \frac{\alpha}{\sqrt T}$ after $T$ iterations of the outer loop. Also, the termination condition for the algorithm is now $\lambda_{\min}(\hat{P}_{S'})\ge 2\frac{\alpha}{\sqrt{T}}$ for some $S'\subseteq [L]$.

Take $S = \argmax_{S\subseteq [L], |S|=s} \lambda_{\min}(P_S)$. By Lemma \ref{lemma:eig_perturbation}, after $T=\frac{9\alpha^2}{\lambda^2_{\min}(P_S)}$ rounds, we have with probability at least $1-L\delta$ that $$\lambda_{\min}(P_S)-\lambda_{\min}(\hat{P}_S)\le \frac{\alpha}{\sqrt{\frac{9\alpha^2}{\lambda^2_{\min}(P_S)}}}= \frac{\lambda_{\min}(P_S)}{3}\implies \lambda_{\min}(\hat{P}_S)\ge \frac{2}{3}\lambda_{\min}(P_S)\ge \frac{2\alpha}{\sqrt{T}}.$$On the other hand, because the algorithm hasnt terminated, we must have $\lambda_{\min}(\hat{P}_S)\le 2\frac{\alpha}{\sqrt{T}}$. Since this contradicts the termination condition, we know that the algorithm terminates in no more than $\frac{9\alpha^2}{\lambda_{\min}(P_S)^2}$ rounds of the outer loop.

Suppose whenever we terminate we output $\tilde{S}\subset [L], \tilde{S}=\argmax_{S'\in [L]}\lambda_{\min}(\hat{P}_{S'})$. Then we have $\lambda_{\min}(P_{S'})\ge \lambda_{\min}(\hat{P}_{S'})-\frac{\alpha}{\sqrt{T}}\ge \frac{1}{2}\lambda_{\min}(\hat{P}_{S'})$. We have 
\begin{align*}
\lambda_{\min}(P_S)
&\le \lambda_{\min}(\hat{P}_S)+\frac{\alpha}{\sqrt{T}} &&\text{by Lemma \ref{lemma:PPT_perturbation} on $P_S$}\\
&\le \lambda_{\min}(\hat{P}_{\tilde{S}})+\frac{\alpha}{\sqrt{T}}&&\text{by $\tilde{S}=\argmax_{S'\in [L]}\lambda_{\min}(\hat{P}_{S'})$}\\
&\le \lambda_{\min}(P_{\tilde{S}})+2\frac{\alpha}{\sqrt{T}}&&\text{by Lemma \ref{lemma:PPT_perturbation} on $P_{\tilde{S}}$}\\
&\le\frac{3}{2}\lambda_{\min}(\hat{P}_{\tilde{S}})&&\text{by termination condition}\\
&\le 3\big(\lambda_{\min}(\hat{P}_{\tilde{S}})-\frac{\alpha}{\sqrt{T}}\big)&&\text{by $\lambda_{\min}(\hat{P}_{\tilde{S}})\ge 2\frac{\alpha}{\sqrt{T}}$ from the termination condition}\\
&\le 3\lambda_{\min}(P_{\tilde{S}})&&\text{by Lemma \ref{lemma:PPT_perturbation} on $P_{\tilde{S}}$}
\end{align*}
In summary, with probability $1-L\delta$, this procedure terminates in no more than $\frac{576L^2M^2R^2(M+R)^2\big(d\log 6+\log\frac{1}{\delta}\big)^2}{\lambda_{\min}^2}$ and outputs $\tilde{S}$ with $\lambda_{\min}(P_{\tilde{S}})\ge \frac{1}{3}\lambda_{\min}$.
\end{proof}

\subsection{known $\lambda_{\min}$, unknown $s$}\label{appendix:unknown_s}
In Algorithm \ref{algorithm:CORE-SET_alternative}, we modify CORE-SET to accept a threshold $\lambda_{\min}$ rather than known $s$ as discussed in Remark \ref{remark:dimension_known}.

\begin{algorithm}[th!]
   \caption{CORE-SET for known $\lambda_{\min}$}
   \label{algorithm:CORE-SET_alternative}
\begin{algorithmic}
\STATE $t \leftarrow 1$\;
\WHILE{$8LR(M+R)\frac{d\log 6+\log\frac{1}{\delta}}{\sqrt{T}}\le \lambda_{\min}$}
    \FOR{$i\in [d]$}
        \FOR{$p\in [L]$}
            \STATE Play $(e_i, \theta_p)$, observe $x$\;
            \STATE Update $\Theta_i$ with $(e_i, x)$\;
        \ENDFOR
    \ENDFOR
    \STATE $t\leftarrow t+1$\;
\ENDWHILE
\STATE $k\leftarrow \min{j: \lambda_{j}(\sum_{i\in [L]}\hat{\theta}_i\hat{\theta}_i^\top)\ge\lambda_{\min}}$\;
\COMMENT{$\lambda_j(M)$ denotes the $j$th smallest eigenvalue of $M$}
\STATE return $(\argmax_{S \subseteq [L], |S|=k} \lambda_{\min}(\sum_{i\in [L]}\hat{\theta}_i\hat{\theta}_i^\top), t)$
\end{algorithmic}
\end{algorithm}
In the same way as before, this ensures that $|\lambda_{i}(\sum_{i\in [L]}\hat{\theta}_i\hat{\theta}_i^\top )-\lambda_{i}(\sum_{i\in [L]}\theta_i\theta_i^\top )|\le \lambda_{\min}$, so a zero singular value cannot be mistaken for non-zero or vice-versa. We compute first the rank of the protected space in $k$ and then find the best subset of protected vectors to span that space.
\section{Regret upper bound for Theorem \ref{thm:known_subspace_dim}}\label{appendix:regretbound}
\begin{theorem}
If we have $\mathcal{A}_t=\mathcal{B}_2^d$, the regret of Algorithm \ref{algorithm:protectedLUCB} satisfies
\small
\begin{align*}
    \mathcal{R}_{[T]} &\le 2\sqrt{((\frac{3\sqrt{s}}{\lambda_{\min}}M+1)^2\beta_T(\frac{\delta}{2(L+1)})+M^2)(s+1)Td\log (1+\frac{TL}{d\rho})}+\frac{1152L^3MR^2(M+R)^2d\big(d\log 6+\log\frac{2L}{\delta}\big)^2}{\lambda_{\min}^2}
\end{align*} with probability $1-\delta$ where $$\sqrt{\beta_t(\delta)}=R\sqrt{d\log((1+\frac{tM^2}{\rho})/\delta)}+M\rho^{\frac{1}{2}}.$$
\end{theorem}
As a reminder, we use the notation $f(a, \{\tilde{\theta}_i\})=\langle a, \Proj_{\{\tilde{\theta}_i\}}^{\perp}\tilde{\theta}_0\rangle$, $(a_t, \{\overline{\theta}_i\})=\argmax_{a\in \mathcal{A}, \{\tilde{\theta}_i \in \Theta_i\}_{i\in \tilde{S}}} f(a, \{\tilde{\theta}_i\})$ and $i_t=\argmax_{i\in \tilde{S}} ||a_t||_{V_{i, t}^{-1}}\sqrt{\beta_{|T_{i, t}|}}$. We take $\tilde{S}$ to be the coreset returned by CORE-SET procedure, which satisfies $\lambda_{\min}(\tilde{S})\ge \frac{1}{3}\lambda_{\min}$. We use $\{y_i\}$ as a shorthand for $\{y_i\}_{i\in \tilde{S}}$.
\begin{property}\label{property:self_adjoint} If $\Proj^{\perp}a=a$ then for all $b$, beacause $\Proj^\perp$ is self-adjoint and idempotent, we have $\langle a, \Proj^{\perp}b\rangle = \langle \Proj^{\perp}a, b\rangle = \langle a, b\rangle$. 
\end{property}
\begin{lemma}\label{lemma:maximizer_perp}
    If $\mathcal{A}$ is $\mathcal{B}_2^d$, $a_t\in \mathcal{B}_2^d$ satisfies $\Proj_{\{\overline{\theta}_i\}}^{\perp}a_t = a_t$.
\end{lemma}
\begin{proof}
Consider the action $a'=\frac{\Proj_{\{\overline{\theta}_i\}}^{\perp}a_t}{||\Proj_{\{\overline{\theta}_i\}}^{\perp}a_t||_2}$. This satisfies 
\begin{align*}
f(a', \{\tilde{\theta}_i\})=\langle a', \Proj_{\{\overline{\theta}_i\}}^{\perp}\overline{\theta}_0\rangle 
&= \frac{1}{||\Proj_{\{\overline{\theta}_i\}}^{\perp}a_t||_2}\langle \Proj_{\{\overline{\theta}_i\}}^{\perp}a_t, \Proj_{\{\overline{\theta}_i\}}^{\perp}\overline{\theta}_0\rangle \\
&=\frac{1}{||\Proj_{\{\overline{\theta}_i\}}^{\perp}a_t||_2}\langle a, \Proj_{\{\overline{\theta}_i\}}^{\perp}\overline{\theta}_0\rangle &&\text{by Property \ref{property:self_adjoint}}\\
&\ge \langle a_t, \Proj_{\{\overline{\theta}_i\}}^{\perp}\overline{\theta}_0\rangle=f(a_t, \{\overline{\theta}\}) &&\text{because $\lVert\Proj_{\{\overline{\theta}_i\}}^{\perp}a_t\rVert_2<1$}
\end{align*}
Because $a_t$ is optimal, we must have equality.
$$||\Proj_{\{\theta_i\}_{i\in [L]}}^{\perp}a_t||_2=||a_t||_2\implies \Proj_{\{\overline{\theta}_i\}}^{\perp}a_t = a_t.$$
\end{proof}

\begin{property}\label{property:proj_is_zero}
For any choice of $\{x_i\}\in \mathbb{R}$, we have $\Proj_{\{\overline{\theta_i}\}}^{\perp}\sum_{i\in \tilde{S}} \overline{\theta}_i x_i=0$
\end{property}
This is true since $\Proj_{\{\overline{\theta_i}\}}^{\perp} \overline{\theta}_i=0$ for all $i\in \tilde{S}$.
\begin{property}\label{property:opt}Because $a_t$ is the optimistic action in Equation~\ref{eqn:optimistic_optimizer}, we have $\langle a_t, \Proj_{\{\overline{\theta_i}\}}^{\perp}\overline{\theta}_0\rangle \ge \langle a_*, \Proj_{\{\theta_i\}}^{\perp}\theta_0\rangle$.
\end{property}

\begin{lemma}\label{lemma:delta_upperbound}
Suppose at time $t$ a player plays $(a_t, i_t)$. Then the suboptimality $\Delta_{a_t}$ is upper bounded as 
$$\Delta_{a_t}\le 2(3\frac{\sqrt{s}}{\lambda_{\min}}M+1)||a_t||_{V_{i_t, t}^{-1}}\sqrt{\beta_{T}(\delta)}$$
\end{lemma}
\begin{proof}
We denote by $V_{i, t}$ a matrix that represents the extent of exploration with the $i$th vector, $V_{i, t} = \sum_{s\le t} \mathbbm{1}_{i_t = i}a_sa_s^T$, and by $V_t$ the covariance of the exploration, $V_t = \sum_{s\le t} a_sa_s^T$. We will denote by $T_{i, t}$ the number of times $i$ has been queried upto and including time $t$, so $T_{i, t}=\sum_{s\le t}\mathbbm{1}_{I_s=i}$.

Because the $\theta_i, i\in \tilde{S}$ are a basis for the protected space, we can write $\Proj_{\{\theta_i\}}\theta_0$ as
\begin{equation}\label{eqn:proxy}\Proj_{\{\theta_i\}}\theta_0 = \sum_{i\in {\tilde{S}}} \theta_i x_i
\end{equation}
for unique $x_i\in \mathbb{R}$.  

The suboptimality of an action is upper bounded as follows: 
\begin{align*}
    \Delta_{a_t} &= \langle a_*-a_t, \Proj_{\{\theta_i\}}^{\perp}\theta_0\rangle\\
    & \le \langle a_t, \Proj_{\{\overline{\theta_i}\}}^{\perp}\overline{\theta}\rangle-\langle a_t, \Proj_{\{\theta_i\}}^{\perp}\theta_0\rangle &&\text{by Property \ref{property:opt}}\\
    & = \langle a_t, \Proj_{\{\overline{\theta_i}\}}^{\perp}(\overline{\theta}-\theta_0)\rangle+\langle a_t, (\Proj_{\{\theta_i\}}-\Proj_{\{\overline{\theta_i}\}})\theta_0\rangle && \text{by $\Proj^\perp x = (I-\Proj)x$}\\
    & = \langle \Proj_{\{\overline{\theta_i}\}}^{\perp}a_t, (\overline{\theta}-\theta_0)\rangle+\langle a_t, (\Proj_{\{\theta_i\}}-\Proj_{\{\overline{\theta_i}\}})\theta_0\rangle &&\text{by Property \ref{property:self_adjoint}}\\
    & = \langle \Proj_{\{\overline{\theta_i}\}}^{\perp}a_t, (\overline{\theta}-\theta_0)\rangle+\langle \Proj_{\{\overline{\theta_i}\}}^\perp a_t, (\Proj_{\{\theta_i\}}-\Proj_{\{\overline{\theta_i}\}})\theta_0\rangle&&\text{by Lemma \ref{lemma:maximizer_perp}}\\
    & = \langle \Proj_{\{\overline{\theta_i}\}}^{\perp}a_t, (\overline{\theta}-\theta_0)\rangle+\langle a_t, \Proj_{\{\overline{\theta_i}\}}^\perp(\Proj_{\{\theta_i\}}-\Proj_{\{\overline{\theta_i}\}})\theta_0\rangle&&\text{by Property \ref{property:self_adjoint}}\\
    & = \langle a_t, (\overline{\theta}-\theta_0)\rangle+\langle a_t, \Proj_{\{\overline{\theta_i}\}}^{\perp}\Proj_{\{\theta_i\}}\theta_0\rangle&&\text{by $\Proj^\perp \Proj x=0$}\\
    &= \langle a_t, (\overline{\theta}-\theta_0)\rangle+\langle a_t, \Proj_{\{\overline{\theta_i}\}}^{\perp}\sum_{i\in \tilde{S}} \theta_i x_i\rangle &&\text{by \eqref{eqn:proxy}}\\
    &= \langle a_t, (\overline{\theta}-\theta_0)\rangle+\langle a_t, \Proj_{\{\overline{\theta_i}\}}^{\perp}\big(\sum_{i\in \tilde{S}} \big(\theta_i-\overline{\theta}_i\big) x_i+\sum_{i\in \tilde{S}} \overline{\theta}_i x_i\big)\rangle \\
    &= \langle a_t, (\overline{\theta}-\theta_0)\rangle+\langle a_t, \Proj_{\{\overline{\theta_i}\}}^{\perp}\sum_{i\in \tilde{S}} \big(\theta_i-\overline{\theta}_i\big) x_i\rangle &&\text{by Property \ref{property:proj_is_zero}}\\
    &= \langle a_t, (\overline{\theta}-\theta_0)\rangle+\langle a_t, \sum_{i\in \tilde{S}} \big(\theta_i-\overline{\theta}_i\big) x_i\rangle &&\text{by Property \ref{property:self_adjoint} and Lemma \ref{lemma:maximizer_perp}}\\
    &= \langle a_t, (\overline{\theta}-\theta_0)\rangle+\sum_{i\in \tilde{S}}\big(\langle a_t,  \theta_i-\overline{\theta}_i\rangle x_i\big)\\
    &\le \langle a_t, (\overline{\theta}-\theta_0)\rangle+\max_{i\in \tilde{S}}\langle a_t, \theta_i-\overline{\theta}_i\rangle \sum_{i\in \tilde{S}} |x_i|\\
    &\le ||a_t||_{V_{0, t}^{-1}} ||\theta_i-\overline{\theta}_i||_{V_{0, t}}+||x_{\tilde{S}}||_1\max_{i\in \tilde{S}}||a_t||_{V_{i, t}^{-1}} ||\theta_i-\overline{\theta}_i||_{V_{i, t}} &&\text{by Cauchy-Schwartz}
\end{align*}
So \begin{equation}\label{eqn:delta_upperbound}\Delta_{a_t}\le ||a_t||_{V_{0, t}^{-1}}||\overline{\theta}-\theta_0||_{V_{0, t}}+\max_{i\in \tilde{S}}||a_t||_{V_{i, t}^{-1}}||\overline{\theta}_i-\theta_i||_{V_{i, t}}||x_{\tilde{S}}||_1\end{equation}From a union bound, we know that with probability $1-(L+1)\delta$, equation \eqref{eq:confidence_set} holds for all $i\in [L]\cup \{0\}$. Because the coreset efficiently spans the protected space, we have
\begin{align*}
||x_{\tilde{S}}||_1
& \le \sqrt{|\tilde{S}|}||x_{\tilde{S}}||_2 &&\text{by Cauchy-Schwartz}\\
&\le \sqrt{s}\frac{1}{\lambda_{\min}(\tilde{S})}||\sum_{i\in {\tilde{S}}} \theta_i x_i||_2 &&\text{by $||\sum_{i\in {\tilde{S}}} \theta_i x_i||_2\ge \lambda_{\min}(\tilde{S})||x||_2$}\\
& = \sqrt{s}\frac{1}{\lambda_{\min}(\tilde{S})}||\Proj_{\{\theta_i\}}\theta_0||_2\\
&\le \sqrt{s}\frac{3}{\lambda_{\min}}||\theta_0||_2&&\text{$\lambda_{\min}(\tilde{S})\ge \frac{1}{3}\lambda_{\min}$ by Theorem \ref{thm:CORE-SET_appendix}, $\Proj$ a contraction}
\end{align*}
So \begin{equation}\label{eqn:l1_bound}||x_{\tilde{S}}||_1\le \sqrt{s}\frac{3}{\lambda_{\min}}||\theta_0||_2\end{equation} The index of the query chosen alongside $a_t$ is chosen to be the one such that for all $t$, we have \begin{equation}\label{eqn:max_uncertainty}||a_t||_{V_{i_t, t}^{-1}}\sqrt{\beta_{T_{i_t, t}}(\delta)} \ge ||a_t||_{V_{i, t}^{-1}}\sqrt{\beta_{T_{i, t}}(\delta)}~\forall i\in \{0\}\cup \tilde{S}\end{equation} 
Geometrically, this is the index corresponding to the vector that is least understood in the chosen direction, since an upper bound on the radius of a confidence ellipsoid for vector $\theta_i$ in direction $a_t$ is given by $||a_t||_{V_{i_t, t}^{-1}}\sqrt{\beta_{T_{i_t, t}}(\delta)}$. 

This allows us to upper bound $\Delta_{a_t}$ in terms of the history as  \begin{align*}
\Delta_{a_t}
&\le ||a_t||_{V_{0, t}^{-1}}||\overline{\theta}-\theta_0||_{V_{0, t}}+\max_{i\in \tilde{S}}||a_t||_{V_{i, t}^{-1}}||\overline{\theta}_i-\theta_i||_{V_{i, t}}||x_{\tilde{S}}||_1&&\text{by \eqref{eqn:delta_upperbound}}\\
&\le 2||a_t||_{V_{0, t}^{-1}}\sqrt{\beta_{T_{0, t}}(\delta)}+2\max_{i\in [L]}||a_t||_{V_{i, t}^{-1}}\sqrt{s\beta_{T_{i, t}}(\delta)}\frac{3}{\lambda_{\min}}||\theta_0||_2 &&\text{by \eqref{eq:confidence_set} and \eqref{eqn:l1_bound}}\\
&\le 2(3\frac{\sqrt{s}}{\lambda_{\min}}M+1)||a_t||_{V_{i_t, t}^{-1}}\sqrt{\beta_{T_{i_t, t}}(\delta)}\le 2(3\frac{\sqrt{s}}{\lambda_{\min}}M+1)||a_t||_{V_{i_t, t}^{-1}}\sqrt{\beta_{T}(\delta)}&&\text{by \eqref{eqn:max_uncertainty}, and $T_{i, t}\le T$}
\end{align*}
\end{proof}
The following is a standard identity for the covariance matrices.
\begin{lemma}\label{lemma:covariance_logdet}
    The $\log \det$ of the covariance matrix satisfies 
    $$\log \det V_{i, T}
    = d\log \lambda+\sum_{t\le T}\mathbbm{1}_{i_t=i}\log (1+||a_t||^2_{V_{i, T_{i, t-1}}^{-1}})$$
\end{lemma}
\begin{proof}
This follows from Sylvester's identity~\cite{sylvester1851xxxvii},  
\begin{align*}
\log \det V_{i, t}
&=\log\det(V_{i, t-1}+a_ta_t^T)\\
&=\log\det V_{i,t-1}+\log\det (I+a_ta_t^TV_{i, t-1}^{-1})\\
&=\log\det V_{i,t-1}+\log\det (I+a_t^TV_{i, t-1}^{-1}a_t)\\
&=\log\det V_{i,t-1}+\log\det (I+||a_t||^2_{V_{i, t-1}^{-1}})
\end{align*}
\end{proof}
There is also a simple upper bound on the $\log\det$ that comes from the arithmetic-geometric means inequality.
\begin{lemma}\label{lemma:logdet_upperbound}
    $$\log \det V_{i, T} \le d\log (\rho+\frac{T_{i, T}L}{d})$$
\end{lemma}
\begin{proof}
Let $\lambda_j(V_{i, t})$ denote the eigenvalues of $V_{i, t}$.
\begin{align*}
    \log \det V_{i, T} 
    = \log \prod_{j\le d} \lambda_j(V_{i, t}) \le d\log \frac{\sum_{j\le d} \lambda_j(V_{i, t})}{d} 
    =d\log \frac{\Tr V_{i, t}}{d}\le d\log (\rho+\frac{|T_i|L}{d}).
\end{align*} 
\end{proof}

\begin{proof}[Proof of Theorem \ref{appendix:regretbound}]
We first get an upper bound for the regrets separately for times in which each of the protected vectors are queried
\begin{align*}
    &\log \det V_{i, T}\\
    &= d\log \lambda+\sum_{t\le T}\mathbbm{1}_{i_t=i}\log (1+||a_t||^2_{V_{i, T_{i, t-1}}^{-1}})&&\text{by Lemma \ref{lemma:covariance_logdet}}\\
    &\ge d\log \lambda+\sum_{t\le T}\mathbbm{1}_{i_t=i}\log (1+\frac{\Delta_{a_t}^2}{4(3\frac{\sqrt{s}}{\lambda_{\min}}M+1)^2\beta_{T}(\delta)})&&\text{by Lemma \ref{lemma:delta_upperbound}}\\
    &\ge d\log \rho + \sum_{t\le T}\mathbbm{1}_{i_t=i}\frac{\Delta_{a_t}^2}{4(3\frac{\sqrt{s}}{\lambda_{\min}}M+1)^2\beta_T(\delta)+4M^2} &&\text{by $\log (1+x)\ge \frac{x}{1+x}$, also $\Delta_{a_t}\le 2M$ }\\
    &\ge d\log \rho + \frac{\big(\sum_{t\le T}\mathbbm{1}_{i_t=i}\Delta_{a_t}\big)^2}{4T_{i, T}((3\frac{\sqrt{s}}{\lambda_{\min}}M+1)^2\beta_T(\delta)+M^2)}&&\text{by Cauchy-Schwartz}
\end{align*}

Writing this as an upper bound on the sub-optimality, we have:
\begin{align*}
    \sum_{t\le T}\mathbbm{1}_{i_t=i}\Delta_{a_t}\le  \sqrt{4T_{i, T}((3\frac{\sqrt{s}}{\lambda_{\min}}M+1)^2\beta_T(\delta)+M^2)\big(\log \det V_{i, T}-d\log \rho\big)}
\end{align*}
We can now combine the regrets from each of the protected vectors:
\begin{align}\label{eqn:regret_upperbound_full}
    \sum_{t\le T} \Delta_{a_t}
    &=\sum_{i\in \{0\}\cup \tilde{S}} \sum_{t\le T} \mathbbm{1}_{I_t=i}\Delta_{a_t}\nonumber\\
    &\le \sum_{i\in \{0\}\cup\tilde{S}}\sqrt{4T_{i, T}((3\frac{\sqrt{s}}{\lambda_{\min}}M+1)^2\beta_T(\delta)+M^2)(\log \det V_{i, T}-d\log \rho)}\nonumber\\
    &\le \sum_{i\in \{0\}\cup\tilde{S}}\sqrt{4T_{i, T}((3\frac{\sqrt{s}}{\lambda_{\min}}M+1)^2\beta_T(\delta)+M^2)d\log (1+\frac{T_{i, T}L}{d\rho})} && \text{by Lemma \ref{lemma:logdet_upperbound}}
\end{align}

In each round, we concentrate our querying only on the vectors in the core set $\tilde{S}$ with $|\tilde{S}|=s$, so we have the first equality below. Next applying Cauchy-Schwartz inequality we obtain the result.
\begin{align*}
    \mathcal{R}_{[T]}
    &=\sum_{i\in \{0\}\cup [L]}\sum_{t\in [T]}\mathbbm{1}_{I_t=i}\Delta_{a_t}=\sum_{i\in \{0\}\cup \tilde{S}}\sum_{t\in [T]}\mathbbm{1}_{I_t=i}\Delta_{a_t}\\
    &\le 2\sqrt{((\frac{3\sqrt{s}}{\lambda_{\min}}M+1)^2\beta_T(\delta)+M^2)(s+1)Td\log (1+\frac{TL}{d\rho})}
\end{align*}
From \cite{yadkori2011OFUL} we know that with probability $1-\delta$, the confidence ellipsoids constructed contain the true parameters for each individual vector for $\beta_T(\delta) = \big(R\sqrt{d\log((1+\frac{TM^2}{\rho})/\delta)}+\rho M^{\frac{1}{2}}\big)^2$. By a union bound, with probability $1-L\delta$, these inequalities hold for all $L$ simultaneously.

Finally, we must add the regret accrued during the initial core-set estimation phase. By Theorem \ref{thm:CORE-SET_appendix}, with probability $1-\frac{\delta}{2}$, this phase lasts at most $\frac{576L^3R^2(M+R)^2d\big(d\log 6+\log\frac{2L}{\delta}\big)^2}{\lambda_{\min}^2}$ rounds, and adds at most $\frac{1152L^3MR^2(M+R)^2d\big(d\log 6+\log\frac{2L}{\delta}\big)^2}{\lambda_{\min}^2}$ to the regret. The rest of the rounds accrue a regret of at most $2\sqrt{((\frac{3\sqrt{s}}{\lambda_{\min}}M+1)^2\beta_T(\frac{\delta}{2(L+1)})+M^2)(s+1)Td\log (1+\frac{TL}{d\rho})}$ with probability $1-\frac{\delta}{2}$. So the total regret is upper bounded by $$2\sqrt{((\frac{3\sqrt{s}}{\lambda_{\min}}M+1)^2\beta_T(\frac{\delta}{2(L+1)})+M^2)(s+1)Td\log (1+\frac{TL}{d\rho})}+\frac{1152L^3MR^2(M+R)^2d\big(d\log 6+\log\frac{2L}{\delta}\big)^2}{\lambda_{\min}^2}$$ with probability at least $1-\delta$.
\end{proof}

\section{Proof of Regret Lower Bound (Theorem~\ref{thm:hardness})}\label{appendix:KL_upperbound}
\begin{theorem}\label{appendix:hardness}
There is an instance of the Protected Linear Bandit problem such that any algorithm incurs a regret of $\Omega(T^{\frac{3}{4}})$.
\vspace{-1em}
\end{theorem}
\begin{proof}
Consider a pair of instances, denoted with superscripts $(1)$ and $(2)$. For both, we set our ambient space to have dimension $d=2$, and set $s=L=1$. We denote by $u_\alpha\in \mathbb{R}^2$ the vector $(\cos \alpha, \sin \alpha)$. Take $\alpha = T^{-\frac{1}{4}}$. We set $\theta^{(1)}_0=\theta^{(2)}_0=u_{\frac{\pi}{2}-\alpha}$. In instance $(1)$, we set $\theta^{(1)}_1=u_0$ while in instance $(2)$, we set $\theta^{(2)}_1=u_{-\alpha}$. In both instances, in each round, we allow the player an action space that consists of either the actions $\{u_{\pi-\alpha}, u_{2\alpha}\}$ or $\{u_{\pi-\alpha}, u_{2\alpha}, u_{\pi-3\alpha}\}$ with equal probability. 

We will denote by $\Pi_1, \Pi_2$ respectively the orthogonal projection for $\theta_1$ in instance $(1)$ and $(2)$. We have $$\Pi_1=\begin{pmatrix} 0&0\\0&1\end{pmatrix}, \quad\quad \Pi_2 = \begin{pmatrix} \sin^2 \alpha&-\frac{1}{2}\sin (2\alpha)\\-\frac{1}{2}\sin (2\alpha)&\cos^2 \alpha \end{pmatrix}.$$ 

Note that for $x<1$, we have $x\ge \sin x\ge \frac{5}{6}x$ and $1\ge 1-\cos x\ge 1-x$, and for all $x$, $\cos x\ge 1-\frac{x^2}{2}$. We will assume $T\ge 256\implies 0\le\alpha\le\frac{1}{4}$.

In the first instance, action $u_{\pi-\alpha}$ has reward $\langle u_{\frac{\pi}{2}-\alpha}, \Pi_1u_{\pi-\alpha}\rangle=\sin \alpha\cos\alpha$, action $u_{2\alpha}$ has reward $\langle u_{\frac{\pi}{2}-\alpha}, \Pi_1u_{2\alpha}\rangle=\sin(2\alpha)\cos \alpha$, and $u_{\pi-3\alpha}$ has reward $\langle u_{\frac{\pi}{2}-\alpha}, \Pi_1u_{\pi-3\alpha}\rangle=\sin(3\alpha)\cos\alpha$. When available, $u_{\pi-3\alpha}$ is the optimal action. When it is not available, $u_{2\alpha}$ is optimal. In either cases, the sub-optimality gap is at least $\cos\alpha (\sin(3\alpha)-\sin(2\alpha))\ge 1-\alpha$. So $\cos\alpha (\sin(3\alpha)-\sin(2\alpha))\ge (1-\frac{2}{\pi}\alpha)(\frac{1}{2}\alpha)\ge\frac{1}{4}\alpha$ for small enough $\alpha\le \frac{1}{4}$.

In the second instance, action $u_{\pi-\alpha}$ has reward $\langle u_{\frac{\pi}{2}-\alpha}, \Pi_2u_{\pi-\alpha}\rangle=0$, action $u_{2\alpha}$ has reward $\langle u_{\frac{\pi}{2}-\alpha}, \Pi_2u_{2\alpha}\rangle=\sin(3\alpha)$, and $u_{\pi-3\alpha}$ has reward $\langle u_{\frac{\pi}{2}-\alpha}, \Pi_2u_{\pi-3\alpha}\rangle=\sin(2\alpha)$. Here, $u_{2\alpha}$ is always optimal. The sub-optimality gap is at least $\sin(3\alpha)-\sin(2\alpha)\ge \frac{1}{2}\alpha$ for $\alpha\le \frac{1}{4}$.

Denote by $A$ the event in which the action $u_{\pi-3\alpha}$ is played fewer than half the times it is observed. Let $\Pr_1(A)$ denote the probability of event $A$ under the distribution induced by interaction of the algorithm with instance $1$, and let $\Pr_2(A^c)$ denote the probability of playing action $u_{\pi-3\alpha}$ at least half the times it is observed under the distribution induced by interaction with instance $(2)$. The regret of the algorithm on instance $(1)$ is then at least $R_1\ge T^{-\frac{1}{4}}\frac{T}{4}\Pr_1(A)$, while the regret on instance $(2)$ is at least $R_2\ge T^{-\frac{1}{4}}\frac{T}{4}\Pr_2(A^c)$. By the Bretagnolle-Huber inequality \cite{LatSze20}, we have $$\Pr_1(A)+\Pr_2(A^c)\ge \frac{1}{2}e^{-D(\Pr_1||\Pr_2)}.$$ So $$R_1+R_2\ge T^{-\frac{1}{4}}\frac{T}{4}\big(\Pr_1(A)+\Pr_2(A^c)\big)\ge \frac{1}{8}T^{\frac{3}{4}}e^{-D(\Pr_1||\Pr_2)}.$$ 
Finally, we can bound $D(\Pr_1||\Pr_2)$ in terms of the KL-divergences of the reward distributions in each round. We have $D(\Pr_1||\Pr_2)\le 21$, as proved below 
\begin{align*}
&D(\Pr_1||\Pr_2) 
= \E_1[\sum_{t\in [T]} D(P_{A^{(1)}_t, i_t, 1}||P_{A^{(1)}_t, i_t, 2})]\\
&= \sum_{t\in [T]} \Pr_1(I^{(1)}_t=0) \E_1[\langle A^{(1)}_t, \theta^{(1)}_{I^{(1)}_t}-\theta^{(2)}_{I^{(1)}_t}\rangle^2 |I^{(1)}_t=0] \\
&\quad+\sum_{t\in [T]} \Pr_1(I^{(1)}_t=1) \E_1[\langle A^{(1)}_t,\theta^{(1)}_{I^{(1)}_t}-\theta^{(2)}_{I^{(1)}_t}\rangle^2| I^{(1)}_t = 1]\\
&= \sum_{t\in [T]} \Pr_1(I^{(1)}_t=1) \E_1[\langle A^{(1)}_t,\theta^{(1)}_{I^{(1)}_t}-\theta^{(2)}_{I^{(1)}_t}\rangle^2| I^{(1)}_t = 1]\\
&\le \sum_{t\in [T]} \E_1[\langle A^{(1)}_t,u_0-u_{-\alpha}\rangle^2] \\
&\le T \langle u_{\pi-3\alpha}, u_0-u_{-\alpha}\rangle^2 \\\
&= \big(\cos(4\alpha)-\cos (3\alpha)\big)^2T\\
&\le \big(\frac{9\alpha^2}{2}\big)^2T\\
&\le 21\alpha^4 T = 21
\end{align*}

Thus we have $$R_1+R_2\ge \frac{1}{8}T^{\frac{3}{4}}e^{-21}$$ for $T\ge 256.$
This means that any algorithm performs poorly on at least one of the two instances $$\max\{R_1, R_2\}\geq \frac{e^{-21}}{16}T^{\frac{3}{4}}.$$

\end{proof}


\section{Baselines}

The pseudocode for various baselines discussed in the Section~\ref{sec:experiment} are provided below.

\subsection{Round Robin, $\epsilon_t$ greedy linUCB: }
\begin{algorithm}[th!]
   \caption{Round Robin, $\epsilon$ greedy linUCB}
   \label{algorithm:RRLinUCB}
\begin{algorithmic}
    \STATE {\bfseries Input:} protected subspace dimension $s$
    \STATE $t\leftarrow 1$\;
    \STATE $V_i\leftarrow \rho I_d$
    \STATE $l \leftarrow 1$
    \FOR{$t \in [T]$}
    \STATE $flag\leftarrow$ Bernoulli$(\epsilon_t)$
    \IF{flag}
    \STATE $l \leftarrow (l+1)\pmod L$
    \STATE $I_t\leftarrow (l+1)$
    \STATE $A_t = \argmax\limits_{a\in \mathcal{A}_t, \tilde{\theta}_{I_t}\in \Theta_{I_t}} \langle a, \tilde{\theta}_{I_t} \rangle$ \label{line:optimization-rr}
    \ELSE
    \STATE $(A_t, \{\overline{\theta}\}_i) = \argmax\limits_{a\in \mathcal{A}_t, \{\tilde{\theta}_i\in \Theta_i\}_{i\in \tilde{S}}, \tilde{\theta}_0\in \Theta_0} \langle a, \Proj_{\{\tilde{\theta}_i\}_{i\in \tilde{S}}}^{\perp}\tilde{\theta}_0 \rangle$\; \label{line:optim-rr-else}
    \STATE $I_t = 0$
    \ENDIF
    \STATE Play $(A_t, I_t)$, observe $X_t$\;
    \STATE Update $\Theta_{I_t}$ with $(A_t, X_t)$
    \STATE $V_{I_t}\leftarrow V_{I_t}+A_tA_t^\top$
    \STATE Increment $T_{I_t}$\;
    \ENDFOR
\end{algorithmic}
\end{algorithm}

\subsection{$\epsilon$ greedy: }
\begin{algorithm}[th!]
   \caption{$\epsilon$ greedy}
   \label{algorithm:EG}
\begin{algorithmic}[1]
    \STATE {\bfseries Input:} protected subspace dimension $s$
    \STATE $t\leftarrow 1$\;
    \STATE $V_i\leftarrow \rho I_d$
    \STATE $l \leftarrow 1$
    \FOR{$t \in [T]$}
    \STATE $flag\leftarrow$ Bernoulli$(\epsilon)$
    \IF{flag}
    \STATE $l \leftarrow (l+1)\pmod L$
    \STATE $I_t\leftarrow $ uniformly random, $A_t \leftarrow$ uniformly random
    \ELSE
    \STATE $l = (l+1) \pmod{L+1}$
    \STATE $P\leftarrow$MLEProjection$(s, \{\hat{\theta}_i\}_{i\in [L]}, \hat{\theta}_0)$
    \STATE $A_t = \argmax\limits_{a\in \mathcal{A}_t} \langle a,  P\hat{\theta_0}\rangle$
    \STATE $I_t = 0$
    \ENDIF
    \STATE Play $(A_t, I_t)$, observe $X_t$\;
    \STATE Update $\Theta_{I_t}$ with $(A_t, X_t)$
    \STATE $V_{I_t}\leftarrow V_{I_t}+A_tA_t^\top$
    \STATE Increment $T_{I_t}$\;
    \ENDFOR
\end{algorithmic}
\end{algorithm}
Here MLEProjection is simply a projection onto the closest $s$ dimensional subspace of $\mathbb{R}^d$ as determined by PCA.

\section{Computing the optimistic parameters}\label{appendix:optimizer}
Let $f(a, P, \theta):=\langle a, \Proj_{\{\theta_i\}_{i\in [L]}}^{\perp}\theta \rangle$. In line 4 of Algorithm \ref{algorithm:protectedLUCB} we must compute 
\begin{equation}\label{eqn:optimistic_optimizer}
(a_t, \overline{P}, \overline{\theta}) = \argmax_{a\in \mathcal{A}_t, P\in \mathcal{P}, \theta^0\in C_{0, t}} f(a, P, \theta^0)
\end{equation}
Note that this is not a concave function. To avoid having to do a grid search, we use a trick that only works exactly if the maximizer is orthogonal to the protected space. This happens, for instance, when the action space is $\mathcal{B}_2^d$ as shown in \ref{lemma:maximizer_perp}.

We can use this observation to construct a lower bound to this function that is tight only for the optimal parameter values. This surrogate can now be used in line 4 of Algorithm \ref{algorithm:protectedLUCB}.
\begin{lemma}
    For any action $a$, we have $\tilde{P}(a) \in \mathcal{P}$, $\tilde{\theta}_0(a) \in \Theta_0$, and $\max_{P\in \mathcal{P}, \theta^0\in \Theta_0} f(a, P, \theta_0) \geq  f(a, \tilde{P}(a), \tilde{\theta_0}(a))$ for
    $$\tilde{\theta}_0(a) :=\hat{\theta}_0+\frac{\sqrt{\beta_{T_{0, t-1}}}a}{||a||_{V_{0, t}}};\quad  \tilde{\theta}_i(a) :=\hat{\theta}_i+(2\alpha_i-1) \frac{\sqrt{\beta_{T_{i, t-1}}}a}{||a||_{V_{i, t}}},$$ where $$\alpha_i = clip_{[0, 1]} \frac{\langle a, \hat{\theta}_i\rangle+\frac{\sqrt{\beta_{T_{i, t-1}}}||a||}{||a||_{V_{i, t}}}}{2\frac{\sqrt{\beta_{T_{i, t-1}}}||a||}{||a||_{V_{i, t}}}}.$$ 
    Moreover, for action $a^*$ that maximizes (\ref{eqn:optimistic_optimizer}), we have $\max_{P\in \mathcal{P}, \theta
    _0\in \Theta_0} f(a^*, P, \theta_0) = f(a^*, \tilde{P}(a^*), \tilde{\theta_0}(a^*))$.
\end{lemma}
\begin{proof}
For any action $a$ the first part hold. For this to be useful, we must first verify that these parameter values are feasible. Indeed,
$$||\theta^0-\hat{\theta}_0||_{V_{0, t-1}}=||\hat{\theta}_0+\frac{\sqrt{\beta_{T_{0, t-1}}}a}{||a||_{V_{0, t}}}-\hat{\theta}_0||_{V_{0, t-1}}=||\frac{\sqrt{\beta_{T_{0, t-1}}}a}{||a||_{V_{0, t}}}||_{V_{0, t-1}}=\sqrt{\beta_{T_{0, t-1}}}$$ and similarly
$$||\theta^i-\hat{\theta}_i||_{V_{i, t-1}}=||\hat{\theta}_i+(2\alpha-1) \frac{\sqrt{\beta_{T_{i, t-1}}}a}{||a||_{V_{i, t}}}-\hat{\theta}_0||_{V_{0, t-1}}=|(2\alpha-1)|||\frac{\sqrt{\beta_{T_{0, t-1}}}a}{||a||_{V_{0, t}}}||_{V_{0, t-1}}\le \sqrt{\beta_{T_{0, t-1}}}$$ for $\alpha\in [0, 1]$. Because this is a feasible point, for a maximization problem, for any action, this is a lower bound to the solution.

At the maximizer $a^*$, we already know that there is a choice of $\tilde{\theta}_i$ for which $\langle a^*, \tilde{\theta}_i\rangle=0$ for all $i\in [L]$. Any other choice of $\tilde{\theta}_i$ that is orthogonal to $a^*$ will result in the same value of $f(a^*, \tilde{P}, \theta_0)$, since $\tilde{\theta}_i$ only acts through a projection collective column space which is orthogonal to $a^*$ anyway. For this choice of $a^*$, and a protected space that is orthogonal to it, the optimal $\tilde{\theta}_0$ is standard.
\end{proof}

\end{document}